\documentclass[nohyperref]{article}

\usepackage[utf8]{inputenc} 
\usepackage[T1]{fontenc}    
\usepackage{hyperref}    
\usepackage{booktabs}        
\usepackage{tabularx}       
\usepackage{amsmath,amsthm,amsfonts}      
\usepackage{nicefrac}       
\usepackage{graphicx}
\usepackage{subfigure}
\usepackage{mathtools}
\usepackage{ctable}

\usepackage{pdfpages}

\usepackage{enumitem}

\definecolor{gainsboro}{rgb}{0.86, 0.86, 0.86}

\newcommand{\expect}[1]{\mathop{{}\mathbb{E}}\left[{#1}\right]}
\newcommand{\condexpect}[2]{\mathbb{E}_{#1}\left[{#2}\right]}

\newcommand{\suchthat}{\ensuremath{~\middle|~}}
\newcommand{\knowing}{\suchthat{}}
\newcommand{\card}[1]{\left\lvert{#1}\right\rvert}
\newcommand{\absv}[1]{\card{#1}}

\newcommand{\norm}[1]{\left\lVert{#1}\right\rVert}

\newcommand{\floor}[1]{\left\lfloor{#1}\right\rfloor}

\newcommand{\indexvar}[3]{\ensuremath{{{#3}^{\ifthenelse{\equal{#1}{}}{}{\left({#1}\right)}}_{#2}}}}
\newcommand{\indexvarNoPar}[3]{\ensuremath{{{#3}^{\ifthenelse{\equal{#1}{}}{}{\left{#1}\right}}_{#2}}}}

\newcommand{\params}[2]{\indexvarNoPar{#1}{#2}{\theta}}

\providecommand{\iprod}[2]{\ensuremath{\left\langle #1,\,#2  \right\rangle}}
\providecommand{\mnorm}[1]{\ensuremath{\left\lvert#1\right\rvert}}
\providecommand{\norm}[1]{\ensuremath{\left\lVert#1\right\rVert }}

\newcommand{\weight}[1]{\params{}{#1}}

\newcommand{\gradient}[2]{\indexvar{#1}{#2}{g}}

\newcommand{\worker}[2]{\indexvar{#1}{#2}{w}}

\newcommand{\proba}[2]{\ensuremath{\text{P}\!\left({#1}\ifthenelse{\equal{#2}{}}{}{\knowing{}{#2}}\right)}}

\DeclareMathOperator*{\argmin}{argmin}

\renewcommand{\paragraph}[1]{\textbf{#1}~}

\newcommand{\drift}[1]{\xi_{#1}}
\newcommand{\dev}[1]{\delta_{#1}}
\newcommand{\mmt}[2]{m^{(#1)}_{#2}}

\newcommand{\AvgMmt}[1]{\overline{m}_{#1}}

\newcommand{\diffmmt}[2]{{\widetilde{m}}^{(#1)}_{#2}}

\usepackage[accepted]{icml2022}

\usepackage[ruled,vlined]{algorithm2e}
\usepackage{microtype}
\usepackage{minitoc}
\usepackage{comment,bm}

\def\D{\mathcal{D}}
\def\R{\mathbb{R}}

\def\P{\mathcal{P}}

\def\H{\mathcal{H}}

\newtheorem{assumption}{Assumption}
\newtheorem{theorem}{\bfseries Theorem}
\newtheorem{theorem**}{\bfseries Theorem}
\newtheorem{theorem*}{\bfseries Theorem}
\newtheorem{lemma*}{\bfseries Lemma}
\newtheorem{proposition}{Proposition}
\newtheorem{proposition*}{Proposition}
\newtheorem{corollary}{Corollary}
\newtheorem{corollary*}{Corollary}
\newtheorem{definition}{Definition}
\newtheorem*{definition*}{Definition}

\newtheorem{remark}{Remark}
\newtheorem{lemma}{Lemma}

\usepackage{chngcntr,etoolbox}
\newcounter{resetdummycounter}

\icmltitlerunning{Byzantine Machine Learning Made Easy}

\begin{document}

\twocolumn[
\icmltitle{Byzantine Machine Learning Made Easy \\ By Resilient Averaging of Momentums} 




\begin{icmlauthorlist}
\icmlauthor{Sadegh Farhadkhani$^*$}{epfl}
\icmlauthor{Rachid Guerraoui$^*$}{epfl}
\icmlauthor{Nirupam Gupta$^*$}{epfl}
\icmlauthor{Rafael Pinot$^*$}{epfl}
\icmlauthor{John Stephan$^*$}{epfl}
\end{icmlauthorlist}

\icmlaffiliation{epfl}{Distributed Computing Laboratory (DCL), School of Computer and Communication Sciences, École Polytechnique Fédérale de Lausanne (EPFL), Switzerland}
\icmlcorrespondingauthor{Nirupam Gupta}{nirupam.gupta@epfl.ch}
\icmlcorrespondingauthor{Rafael Pinot}{rafael.pinot@epfl.ch}

\doparttoc 
\faketableofcontents 

\icmlkeywords{Machine Learning, ICML}

\vskip 0.3in
]



\printAffiliationsAndNotice{\icmlEqualContribution} 

\begin{abstract}
Byzantine resilience emerged as a prominent topic within the distributed machine learning community.
Essentially, the goal is to enhance distributed optimization algorithms, such as distributed SGD, in a way that guarantees convergence despite the presence of some misbehaving (a.k.a., {\em Byzantine}) workers. 
Although a myriad of techniques addressing the problem have been proposed, the field arguably rests on fragile foundations. These techniques are hard to prove correct and rely on assumptions that are (a) quite unrealistic, i.e., often violated in practice, and (b) heterogeneous, i.e., making it difficult to compare approaches. 

We present \emph{RESAM (RESilient Averaging of Momentums)}, a unified framework that makes it simple to establish optimal Byzantine resilience, relying only on standard machine learning assumptions. Our framework is mainly composed of two operators: \emph{resilient averaging} at the server and \emph{distributed momentum} at the workers. We prove a general theorem stating the convergence of distributed SGD under RESAM. Interestingly,  demonstrating and comparing the convergence of many existing techniques become direct corollaries of our theorem, without resorting to stringent assumptions. We also present an empirical evaluation of the practical relevance of RESAM.
\end{abstract}




\section{Introduction}\label{sec:intro}
The vast amount of data collected every day, combined with the increasing complexity of machine learning models, has led to the emergence of distributed learning schemes~\cite{Tensorflow2015,OpenProblemsinFed2021}. In the now classical parameter server distributed architecture, the learning procedure consists of multiple data owners (or {\em workers}) collaborating to build a global model with the help of a central entity (the \emph{parameter server}), typically using the celebrated distributed stochastic gradient descent (SGD) algorithm~\cite{tsitsiklis1986distributed,bertsekas2015parallel}. The server essentially maintains an estimate of the model parameter which is updated iteratively using the \emph{average} of the~{\em stochastic gradients} computed by the workers. 


Nevertheless, this algorithm is vulnerable to "misbehaving" workers that could (either intentionally or inadvertently) 
sabotage the learning by sending arbitrarily bad gradients to the server~\cite{feng2015distributed, su2016fault}. These workers are commonly referred to as \emph{Byzantine}~\cite{ByzProblem}. To address this critical issue, a large body of research has been devoted to adapting distributed SGD to make it converge despite the presence of (a fraction of) Byzantine workers, e.g.,~\cite{krum,chen2017distributed,brute_bulyan,yin2018byzantine,meamed, alistarh2018byzantine, diakonikolas2019sever, allen2020byzantine, prasad2020robust, Karimireddy2021}. The general idea consists in replacing the averaging step of the algorithm with a {\em robust aggregation rule}, basically seeking to filter out the bad gradients. 

Demonstrating the correctness of the resulting algorithms reveals however very challenging, and previous works rely on unusual assumptions. For instance, a large body of work assumes stochastic gradients that follow a specific distribution, e.g., sub-Gaussian/exponential~\cite{chen2017distributed,feng2017onlinelearning, yin2018byzantine, prasad2020robust}. Some approaches rely on stronger assumptions that are not even satisfied by a Gaussian distribution, such as {\em almost surely absolutely boundedness}~\cite{alistarh2018byzantine,diakonikolas2019sever,allen2020byzantine}, or {\em vanishing variance}~\cite{krum,meamed,brute_bulyan,collaborativeElMhamdi21}. Indeed, these assumptions are often violated in practice, resulting in the failure of these approaches when some workers behave maliciously~\cite{little,empire}. Ultimately, the considerable difference in these assumptions from one approach to another makes it quite difficult to compare the underlying techniques.  

In short, whilst Byzantine resilience is considered crucial to establish robustness in distributed machine learning, the field arguably rests on fragile foundations. 

\subsection{Our Contributions}
We present {\bf RESAM} {\em (RESilient Averaging of Momentums)}, a general framework for studying Byzantine resilience in distributed machine learning under minimal assumptions: (1) {\em unbiased} stochastic gradients with {\em bounded variance} and (2) first-order {\em Lipschitz smoothness}.\footnote{These assumptions are elementary for analyzing SGD, even in the non-Byzantine setting~\cite{bottou2018optimization}, and are used in all prior works on Byzantine resilience.} RESAM integrates two main components within distributed SGD, namely \emph{resilient averaging} and \emph{distributed momentum}.

 \begin{enumerate}[label=(\alph*), leftmargin=*]
     \item We introduce resilient averaging as a new elementary criterion of robustness for aggregation rules. It can be verified in an off-line manner and is readily satisfied by many existing schemes, under classical assumptions. It also standardizes the way to measure the robustness of aggregation rules through a parameter $\lambda$, that we call the {\em resilience coefficient}.
     \item We make use of distributed momentum which adapts the notion of gradient momentum~\cite{momentum} to distributed architectures. Specifically, at each step of the algorithm, honest (i.e., non-Byzantine) workers send the momentums of their stochastic gradients to the server, instead of simply sending their gradients.
 \end{enumerate}

\textbf{Byzantine resilience.} We prove a general theorem establishing finite-time convergence of distributed SGD enhanced through RESAM. As an immediate corollary, we make the following contributions.

\begin{enumerate}[label=(\alph*), leftmargin=*]
    \item We show (for the first time) the Byzantine resilience of several existing schemes, without resorting to non-standard assumptions. Our result holds as long as the Byzantine workers represent less than $1/2$ of the system, which is optimal~\cite{alistarh2018byzantine}. 
    \item  We precisely characterize the convergence rates of these schemes through our framework, enabling comparison of their performances on a common theoretical ground. Essentially, our analysis indicates that using aggregation rules with smaller resilience coefficient $\lambda$ results in faster convergence.
\end{enumerate}



\textbf{Technical significance.}
A key observation that enables us to prove our theorem is that the momentums of honest workers' gradients \emph{converge toward one another} as the learning algorithm proceeds. This significantly mitigates the impact of Byzantine workers when using a resilient averaging rule. The caveat is that the conventional techniques used for analyzing the convergence of SGD do not readily apply, since the honest workers' momentums {\em deviate} from the true gradient. To overcome this challenge, we devise a proof technique based on a {\em novel Lyapunov function} which is also of independent interest to the optimization community.

\textbf{Practical relevance.} We report on a comprehensive set of experiments evaluating RESAM on benchmark image classification tasks: MNIST, Fashion-MNIST, and CIFAR-10. We simulate Byzantine behavior using $4$ state-of-the-art attacks. We observe that the algorithm works best when combining resilient averaging and distributed momentum, but performs poorly against some attacks when 
using only one of these notions. This advocates that the combination proposed by RESAM is critical to Byzantine resilience. 

\subsection{Closely Related Work}\label{sec:related work}
We present below comparisons to closely related work. 

\textbf{Resilient averaging.}  
Whilst the robustness criterion of \emph{C-averaging agreement} introduced in~\cite{collaborativeElMhamdi21} shares similarities with our notion of 
resilient averaging, it is studied under the non-standard assumption of \emph{vanishing variance} of the stochastic gradients (and without exploiting the power of distributed momentum). 
Our notion of resilient averaging should also not be confused with the notion of resilience introduced by~\cite{steinhardt2018resilience}, for the latter is an assumption on the distribution of honest workers' gradients. Our notion, on the other hand, is a criterion that can be satisfied by an aggregation rule regardless of the distribution of the workers' gradients.

\textbf{Distributed momentum.}
The first paper to discuss the usefulness of distributed momentum for boosting Byzantine resilience in distributed machine learning is~\cite{distributed-momentum}. Essentially, the paper observes through an extensive set of experiments that distributed momentum helps \emph{some} robustness techniques counter two state-of-the art attacks, namely \emph{little}~\cite{little} and \emph{empire}~\cite{empire}. However, the work lacks concrete theoretical explanations. 
Moreover, our experimental findings go beyond~\cite{distributed-momentum} by considering a wider range of attacks and robustness techniques. 
Another related work~\cite{Karimireddy2021} attempts to formally demonstrate that distributed momentum grants provable Byzantine resilience to the robustness technique they devise, called {\em centered clipping} (CC). While the proof relies on standard assumptions, the algorithm requires prior knowledge on the variance of the gradients, which is quite impractical.
Furthermore, their result only holds for small fractions of Byzantine workers less than $1/9.7$, which is clearly sub-optimal.

\subsection{Paper Outline}
Section~\ref{sec:ModelSetting} formally presents the problem of Byzantine resilience in distributed learning. Section~\ref{sec:KeyConcept} introduces RESAM. Section~\ref{sec:ConvergenceAnalysis} presents our main theorem and its corollary showing resilience of some prominent existing approaches. Section~\ref{sec:experiments} presents our experimental results. Section~\ref{sec:concusion} provides additional related work and discussions. Due to space constraints, we defer proofs to appendices ~\ref{sec:resultsskeleton}, ~\ref{app:proofs}, and~\ref{app:resilience_coefficient}.

\section{Problem Statement}\label{sec:ModelSetting}
We consider the parameter server architecture with $n$ workers $\worker{}{1}, \dots,\worker{}{n}$, and a trusted central server. The workers only communicate with the server and there is no inter-worker communication. We let $\D$ be an unknown data distribution. For a given parameter $\weight{} \in \R^d$, a data point $x \sim \D$ has a real-valued loss function $q(\weight{}, \, x)$. The server aims to compute, by collaborating with the workers, a parameter $\weight{}^{*}$ minimizing the expected loss function $Q(\weight{})$ defined to be
\begin{align}
    Q(\weight{}) = \condexpect{x \sim \D}{q(\weight{}, \, x)} \quad \forall \weight{} \in \R^d. \label{eqn:exp_loss}
\end{align}
We assume $Q$ to be differentiable and to have a minimum, i.e., $\min_{\weight{} \in \R^d} Q(\weight{})$ exists and has a finite value. However, as the loss function $Q$ could be non-convex, e.g., when considering deep neural networks, solving the above optimization problem may be NP-hard~\cite{boyd2004convex}. Thus, a more reasonable goal is to compute a critical point of $Q$, i.e., $\weight{}^*$ such that $\norm{\nabla Q(\weight{}^*)} = 0$ where $\nabla Q$ denotes the {\em gradient} of $Q$ and $\norm{\cdot}$ the Euclidean norm on $\mathbb{R}^d$.

\subsection{Vanilla Distributed SGD} 
The traditional way to solve this learning problem is through a distributed implementation of the classical stochastic gradient descent (SGD) method~\cite{bertsekas2015parallel}. This is an iterative algorithm where, in each step $t$, the server maintains a parameter vector $\weight{t}$ which is broadcast to all the workers. Each worker $\worker{}{i}$ then returns an {\em unbiased} stochastic estimate $g^{(i)}_t$ of the gradient $\nabla Q (\weight{t})$. Specifically, 
\begin{align}
    \gradient{i}{t} = \nabla Q(\weight{t}) + u^{(i)}_t,\label{eqn:grad_i}
\end{align}
where $u^{(i)}_t$ is the realization of a random vector $U(\weight{t})$, defined over $\R^d$, that characterizes the \emph{noise} in the gradient computation at $\weight{t}$.\footnote{The noise $U(\weight{t})$ is usually assumed to be a result of sampling data points from $\D$. However, to keep our discussion more general, we let $U(\weight{t})$ follow any distribution subject to Assumption~\ref{asp:bnd_var}.} Ultimately, the server updates $\weight{t}$ by using the average of the received gradients as follows,
\begin{align}
    \weight{t+1} = \weight{t} - \gamma_t \, \frac{1}{n} \sum_{i = 1}^n \gradient{i}{t}, \label{eqn:naive_SGD}
\end{align}
where $\gamma_t \geq 0$ is referred to as the {\em learning rate} at step $t$. 

\subsection{Classical Assumptions} When all the workers are honest, i.e., they follow the prescribed instructions correctly, the above iterative algorithm provably converges to a critical point of function $Q$, under the following assumptions.

\begin{assumption}[Lipschitz smooth loss function]
\label{asp:lip}
There exists $L < \infty$ such that for all $\weight{}, \, \weight{}' \in \R^d$, $$\norm{\nabla Q(\weight{}) - \nabla Q(\weight{}')} \leq L \norm{\weight{} - \weight{}'}.$$
\end{assumption}

\begin{assumption}[Unbiased gradients with bounded variance]
\label{asp:bnd_var}
For all $\weight{} \in \R^d$, the random vector $U(\theta)$ characterizing the gradient noise at $\theta$ is such that $\expect{U(\weight{})} = 0$, and there exists $\sigma < \infty$ such that $\expect{\norm{U(\weight{})}^2} \leq \sigma^2$.
\end{assumption}

These assumptions are indeed satisfied in many learning problems~\cite{ghadimi2013stochastic, bottou2018optimization}.

\subsection{Byzantine Resilience}
\label{sec:ModelSettingResilience}
We study a scenario where up to $f$ workers of \emph{unknown identities} may be {\em Byzantine}~\cite{ByzProblem}. Such workers may send arbitrarily incorrect information to the server, preventing it from solving the learning problem~\cite{su2016fault}. 
The goal is then to design a learning algorithm that computes a critical point despite the fact that a fraction of the workers may be Byzantine. Formally, given $f$ and a real value $\epsilon > 0$, we aim to design an {\em $(f, \, \epsilon)$-resilient} algorithm, as defined below.


\begin{definition}[{\bf $(f, \, \epsilon)$-Resilience}]
\label{def:resilience}
A distributed learning algorithm is said to be {\em $(f, \, \epsilon)$-resilient} if, despite the presence of up to $f$ Byzantine workers, it enables the server to output a learning parameter $\widehat{\weight{}}$ such that $$\expect{\norm{\nabla Q\left(\widehat{\weight{}} \right)}^2} \leq \epsilon, $$ where $\expect{\cdot}$ is defined over the randomness of the algorithm. Moreover, an algorithm is said to be \emph{optimally resilient} if it is $(f, \, \epsilon)$-resilient for any $f < n/2$ and $\epsilon >0$.
\end{definition}

A standard approach to confer Byzantine resilience to distributed SGD is to replace the simple averaging of the workers' gradients at the server by a more sophisticated aggregation rule that seeks to mitigate the adversarial impact of any incorrect information sent by the Byzantine workers. In particular, consider an aggregation rule $F: \R^{d \times n} \to \R^d$. Then, at every step $t$ the server updates $\weight{t}$ as follows:
\begin{align}
    \weight{t+1} = \weight{t} - \gamma_t \, F\left( \gradient{1}{t}, \ldots, \, \gradient{n}{t}\right).
\end{align} 
Note that the gradient $\gradient{i}{t}$ of any Byzantine worker $\worker{}{i}$ need not follow~\eqref{eqn:grad_i} and may take arbitrary values.

\begin{table*}[t]
\centering
\bgroup
\def\arraystretch{1.5}
\begin{tabular}{c||c|c|c|c|c|c||c} 
{\bf Aggregation rule} & MDA & CWTM & MeaMed & Krum${^*}$ & GM  & CWMed & \textbf{Lower bound} \\ \specialrule{.1em}{.05em}{.05em} 
$\lambda$ & $\frac{2f}{n-f}$ & $\frac{f}{n-f}\Delta$ & $\frac{2f}{n-f}\Delta$ & $1+\sqrt{\frac{n-f}{n-2f}}$ & $1+\frac{n-f}{\sqrt{(n-2f)n}}$ & $\frac{n}{2(n-f)}\Delta$ & $\frac{f}{n-f}$  
\vspace{0.3cm}
\end{tabular}
\egroup
$\Delta \coloneqq \min \{2\sqrt{n-f}, \sqrt{d}\}$
\caption{Resilience coefficients $\lambda$ for various aggregation rules satisfying Definition~\ref{def:rational}, when $f < n/2$. Note that the lower bound for $\lambda$ is $\nicefrac{f}{n-f}$. Thus, MDA has an order-optimal coefficient (differs from the lower bound only by a constant factor).}
\label{tab:resilience_coefficient}
\end{table*}

\textbf{Some notable aggregation rules.}  
In this paper, we consider a wide range of aggregation rules: Krum${^*}$~\citeyear{krum},\footnote{Krum${^*}$ is a variant of Krum, described in Appendix~\ref{sec:Krum}.} \emph{geometric median} (GM)~\citeyear{chen2017distributed}, \emph{minimum diameter averaging} (MDA)~\citeyear{brute_bulyan}, \emph{coordinate-wise trimmed mean} (CWTM)~\citeyear{yin2018byzantine}, \emph{coordinate-wise median} (CWMed)~\citeyear{yin2018byzantine}, \emph{mean-around-median} (MeaMed)~\citeyear{meamed}, \emph{centered clipping} (CC)~\citeyear{Karimireddy2021}, and \emph{comparative gradient elimination} (CGE)~\citeyear{gupta2021byzantine}. We refer the interested reader to Appendix~\ref{app:resilience_coefficient} for a detailed description of these aggregation rules.


\section{RESilient Averaging of Momentums}\label{sec:KeyConcept}
Our framework incorporates the notions of {\em resilient averaging} and {\em distributed momentum} in distributed SGD. 
We first recall distributed momentum, followed by the introduction of resilient averaging. Finally, we present the skeleton of a learning algorithm within RESAM. 
\subsection{Distributed Momentum}\label{sub:dist_mmt}
At each step $t$ of this scheme, upon receiving the current learning parameter $\weight{t}$ from the server, each honest worker $\worker{}{i}$ returns the {\em Polyak's momentum} of its stochastic gradient~\cite{momentum}. This momentum is defined as 
\begin{align}
    \mmt{i}{t} = \beta \mmt{i}{t-1} + (1 - \beta) \gradient{i}{t}, \label{eqn:mmt_i}
\end{align}
where $\mmt{i}{0} = 0$ by convention, $\beta \in [0, \, 1)$, and $\gradient{i}{t}$ is as defined in~\eqref{eqn:grad_i}. We refer to $\beta$ as the {\em momentum coefficient}. Recall that for a Byzantine worker $\worker{}{i}$, the momentum $\mmt{i}{t}$ need not follow~\eqref{eqn:mmt_i}. Upon receiving workers' momentums, the server applies the aggregation rule $F$ to update the parameter $\weight{t}$. Specifically, the server computes
\begin{align}
    \weight{t+1} = \weight{t} - \gamma_t F\left(\mmt{1}{t}, \ldots, \, \mmt{n}{t} \right). \label{eqn:SGD}
\end{align}

\begin{remark}
Distributed momentum differs from its centralized counterpart in that the momentum operation in the former is performed by the workers, unlike in the latter where it is applied by the server after aggregating the gradients.
\end{remark}

\subsection{Resilient Averaging}\label{sub:resilient_avg}
The idea behind the notion of resilient averaging is to ensure that the distance between the result of the aggregation rule and the average of honest workers' momentums is bounded by their {\em diameter} times a factor $\lambda$. We refer to $\lambda$ as the {\em resilience coefficient}. Essentially, smaller the $\lambda$ better the resilience.
We formally define this notion below. 

\begin{definition}[{\bf $(f, \, \lambda)$-Resilient averaging}]
\label{def:rational}
For $f < n$ and real value $\lambda \geq 0$, an aggregation rule $F$ is called {\em $(f, \, \lambda)$-resilient averaging} if for any collection of $n$ vectors $x_1, \ldots, \, x_n$, and any set $S \subseteq \{1, \ldots, \, n\}$ of size $n-f$,
\begin{align*}
    \norm{F(x_1, \ldots, \, x_n) - \overline{x}_S} \leq \lambda \max_{i, j \in S} \norm{x_i - x_j}
\end{align*}
where $\overline{x}_S \coloneqq \frac{1}{\mnorm{S}} \sum_{i \in S} x_i$, and $\mnorm{S}$ is the cardinality of $S$.
\end{definition}

\textbf{Salient features.} Resilient averaging is a simple robustness criterion that is verifiable in an off-line manner, i.e., independently of the dynamics of the learning algorithm. Moreover, this criterion is so elementary that it can be satisfied by a wide class of state-of-the-art aggregation rules under only standard assumptions. This makes it possible to study and compare their resilience properties on a common theoretical ground. Indeed, we show (in Proposition~\ref{prop:resilience_coeffs} below) that all the aggregation rules mentioned in Section~\ref{sec:ModelSettingResilience} satisfy this criterion, except CC and CGE that we discuss separately. \vspace{0.1cm}


\begin{proposition*}
\label{prop:resilience_coeffs}
Consider an aggregation rule $F \in \{\textnormal{MDA}, \textnormal{CWTM}, \textnormal{MeaMed}, \textnormal{Krum$^*$}, \textnormal{GM}, \textnormal{CWMed}\}$. For any $f < n/2$, there exists a resilience coefficient $\lambda$ for which $F$ is $(f, \, \lambda)$-resilient averaging.
\end{proposition*}

We list in Table~\ref{tab:resilience_coefficient} the respective values of $\lambda$ for several aggregation rules satisfying Definition~\ref{def:rational}. Formal derivations of these coefficients can be found in Appendix~\ref{app:resilience_coefficient}. It is worth noting that an $(f,\lambda)$-resilient averaging rule cannot have a resilience coefficient smaller than $\nicefrac{f}{n-f}$ (Lower bound in Table~\ref{tab:resilience_coefficient}). Accordingly, the resilience coefficient we compute for MDA is \emph{order-optimal}, i.e., it differs from the lower bound by a constant factor.



\textbf{Sanity check.}
When the inputs of the honest workers are identical, the output of an $(f, \, \lambda)$-resilient averaging rule is equal to their inputs from Definition~\ref{def:rational} (as the diameter of at least $n-f$ inputs is null). 
This simple yet important sanity check guarantees that when the gradients of honest workers are computed without uncertainty (i.e., $U(\theta)$ is null for all $\theta \in \mathbb{R}^d$) the aggregation rule mimics the majority voting scheme, which is known to be optimal when there is no uncertainty in the correct responses~\cite{lynch1996distributed}. Note that satisfying this sanity check is a necessary condition for being $(f, \, \lambda)$-resilient averaging. 

\textbf{The cases of CGE and CC.} When studying existing rules we encountered two special cases, namely CGE and CC. While CGE clearly does not satisfy the condition of resilient averaging, CC may only satisfy it approximately. 
Besides, CC uses a {\em clipping parameter} that requires a priori knowledge on $\sigma$, and an initial guess on the average of the honest vectors with {\em known} bounded error. These are impractical requirements that are not needed by other rules we consider. As it is unclear whether CC can satisfy our definition under the classical assumptions, in the remaining we adopt an agnostic point of view assuming that it does not.

\subsection{Skeleton of an Algorithm within RESAM} The overall learning procedure combining distributed momentum and a resilient averaging rule is captured in Algorithm~\ref{algo}, presented below.

\begin{algorithm}[ht!]
\SetAlgoLined
{\bf Initialization:} \textcolor{blue}{\bf Server} chooses an arbitrary initial parameter vector $\weight{1} \in \R^d$, a set of $T$ learning rates $\{\gamma_1,\dots,\gamma_T\}$, a deterministic aggregation rule $F: \R^{d \times n} \to \R^d$, and sends the momentum coefficient $\beta \in [0, \, 1)$ to all the workers. Each \textcolor{blue}{\bf honest worker $\worker{}{i}$} sets its initial momentum $\mmt{i}{0} = 0$.\\ \vspace{0.3cm}

{\bf Algorithm's body:}
In each {\bf step $t=1,\dots, \, T$}. \\
\begin{enumerate}[]
    \item \textcolor{blue}{\bf Server} broadcasts $\weight{t}$ to all workers.
    \item Each \textcolor{blue}{\bf honest worker $\worker{}{i}$} sends to the server the momentum $\mmt{i}{t}$ defined by~\eqref{eqn:mmt_i}, i.e., $\mmt{i}{t} = \beta \mmt{i}{t-1} + (1 - \beta) \gradient{i}{t}$ where $\gradient{i}{t}$ is a stochastic gradient as defined in~\eqref{eqn:grad_i}. \\
        \textcolor{teal}{\it (A Byzantine worker $\worker{}{i}$ may send an arbitrary value for its "momentum" $\mmt{i}{t}$.)}
    \item \textcolor{blue}{\bf Server} updates the parameter vector as per~\eqref{eqn:SGD}, i.e., $\weight{t+1} = \weight{t} - \gamma_t \, F\left(\mmt{1}{t}, \ldots, \, \mmt{n}{t} \right)$.
\end{enumerate}
{\bf Output:} \textcolor{blue}{\bf Server} outputs a learning parameter $\widehat{\weight{}}$ chosen randomly from the set $\{\weight{1}, \ldots, \, \weight{T}\}$.
\caption{Distributed SGD using distributed momentum and an $(f,\lambda)$-resilient averaging rule $F$}
\label{algo}
\end{algorithm}

\begin{table*}[htb!]
\centering
\bgroup
\def\arraystretch{2}
\begin{tabular}{c||c|c|c|c|c|c} 
\textbf{Aggregation rule}  & MDA  & CWTM  & MeaMed & Krum${^*}$ & GM & CWMed  \\ 
\specialrule{.1em}{.05em}{.05em} 
$ \epsilon \in  \mathcal{O}\left( \sqrt{ \frac{\sigma^2}{T} \left( \frac{1}{n-f} + \textcolor{blue}{\boldsymbol{\kappa}} \right)} \right) $& $\frac{f^2}{(n-f)}$ & $\frac{f^2}{(n-f)} \, \Delta^2$ & $\frac{f^2}{(n-f)} \, \Delta^2$ & $ \frac{(n-f)^2}{n-2f}$ & $ \frac{(n-f)^3}{(n-2f)n}$  & $ \frac{n^2}{(n-f)} \, \Delta^2 $ 
\vspace{0.3cm}
 \end{tabular}
 \egroup
 $\vartheta \in (0, \, 1) ~  \text{ and } ~ \Delta \coloneqq \min \{2\sqrt{n-f}, \sqrt{d}\}$
 \caption{Rates of convergence, as defined in Corollary~\ref{cor:resilience}, for several $(f, \lambda)$-resilient averaging rules when $f < n/2$. Note that the rates only differ in the $\textcolor{blue}{\boldsymbol{\kappa}}$ term. For simplicity, we only present the values of $\textcolor{blue}{\boldsymbol{\kappa}}$ for the different rules.}
 \label{tab:epsilon}
\end{table*}

\vspace{-0.2cm}
\section{General Convergence Theorem}\label{sec:ConvergenceAnalysis}
We present below our main technical result demonstrating the convergence of Algorithm~\ref{algo} when up to $f$ workers may be Byzantine. Then, as an immediate corollary, we derive the $(f,\epsilon)$-resilience property of the algorithm. Formal proofs of the results are deferred to appendices~\ref{sec:resultsskeleton} and~\ref{app:proofs}.


\subsection{Formal Statements}
\label{sub:main_conv}
We first present our main result in Theorem~\ref{thm:main_conv} below. Essentially, we analyze Algorithm~\ref{algo} upon assuming a sufficient small constant learning rate $\gamma_t$ for all steps $t$, provided that assumptions~\ref{asp:lip} and~\ref{asp:bnd_var} hold true. For simplified presentation of our formal results, we introduce the following notation.
 \begin{align*}
     Q^* & \coloneqq \min_{\weight{} \in \R^d} Q(\weight{}), \\
     a_o & \coloneqq 4 \left(2 \left(Q(\weight{1}) - Q^* \right) + \frac{1}{8L} \norm{\nabla Q \left( \weight{1} \right)}^2 \right), \\
     a_1 & \coloneqq 6912L, \quad \text{ and } \quad a_2 \coloneqq 288L.
 \end{align*}

\noindent \noindent \fcolorbox{black}{gainsboro!10}{
\parbox{0.47\textwidth}{\centering
\begin{theorem}
\label{thm:main_conv}
Consider Algorithm~\ref{algo} with an $(f,\lambda)$-resilient averaging rule and a constant learning rate $\gamma$, i.e., $\gamma_t = \gamma, \, \forall t$ where
\begin{align*}
    \gamma =  \left( \sqrt{\frac{a_o (n-f)}{a_1 \lambda^2 (n-f)^2 + a_2} } \right) \frac{1}{\sigma \sqrt{T}}. 
\end{align*}
If $T \geq \frac{a_o L }{12 \sigma^2 \lambda^2 (n - f)}$ and $\beta = \sqrt{1 - 24 \gamma L} $, then
\begin{align*}
   & \expect{\norm{\nabla Q\left(\widehat{\weight{}} \right)}^2} \leq 2 \sqrt{ \left( a_1 \lambda^2 (n - f) + \frac{ a_2}{n-f}\right)  \frac{a_o \sigma^2}{T} } \\
   & + \left( \frac{a_2 \sigma}{n-f}\right) \left( \sqrt{\frac{a_o (n-f)}{a_1 \lambda^2 (n-f) + a_2} } \right) \frac{1}{T^{\nicefrac{3}{2}}}.
\end{align*}
\end{theorem}
}}


\begin{proof}[Idea of the Proof] Recall that $\mmt{i}{t}$ denotes the momentum of worker $\worker{}{i}$ at step $t$. Below, we denote by $\AvgMmt{t}$ the average momentum of all the honest workers at step $t$. Our proof of Theorem~\ref{thm:main_conv} rests on two key observations, detailed below.

\begin{enumerate}[label=(\alph*),leftmargin=*]
    \item At every step $t$ of Algorithm~\ref{algo}, the growth of the loss function (i.e., $Q(\weight{t+1}) - Q(\weight{t})$) depends positively on both the {\em drift} of each honest worker $\worker{}{i}$ (i.e., $\mmt{i}{t} - \AvgMmt{t}$) and the {\em deviation} of the honest workers from the true gradient (i.e., $\AvgMmt{t} - \nabla Q(\weight{t})$). Essentially, to prove convergence, we need the accumulation of both the drift and the deviation to be inversely proportional to $T$, when scaled by the learning rate $\gamma$.
    \item Upon analyzing these two quantities separately, we observe that whilst increasing the momentum coefficient $\beta$ decreases the accumulation of drift, it increases the accumulation of deviation. Hence, we need to carefully determine an appropriate value for $\beta$ to establish the convergence of Algorithm~\ref{algo}. However, the traditional Lyapunov function of $\expect{Q(\weight{t})}$ turns out to be inadequate for solving this problem. 
\end{enumerate}

To address this issue, we devise a novel Lyapunov function
\[V_t \coloneqq \expect{2 Q(\weight{t}) + \frac{1}{8L^2} \norm{\AvgMmt{t} - \nabla Q(\weight{t})}^2}.\]
By analyzing the growth of $V_t$ along the steps of Algorithm~\ref{algo}, we show that setting the momentum coefficient $\beta = \sqrt{1 - 24 \gamma L} $ yields the stated finite-time convergence. Note that this momentum coefficient is well defined (i.e., it belongs to $[0, \, 1)$) as soon as $T \geq \frac{a_o L }{12 \sigma^2 \lambda^2 (n - f)}$, which explains the condition on $T$ in Theorem~\ref{thm:main_conv}.
\end{proof}

Using Theorem~\ref{thm:main_conv}, we can show that Algorithm~\ref{algo} is $(f, \, \epsilon)$-resilient.
Specifically, by ignoring the higher-order term in $T$, and the constants, we obtain the following corollary.

\begin{corollary}\label{cor:resilience}
Suppose that assumptions~\ref{asp:lip} and~\ref{asp:bnd_var} hold true. Then, Algorithm~\ref{algo} with an $(f,\lambda)$-resilient averaging rule, and parameters $\gamma_t$, $T$ and $\beta$ as defined in Theorem~\ref{thm:main_conv}, is
is $(f,\epsilon)$-resilient with
\[  \epsilon \in  \mathcal{O}\left( \sqrt{ \frac{\sigma^2}{T} \left( \frac{1}{n-f} + \lambda^2 (n-f) \right)} \right).\]
\end{corollary}

Basically, we can obtain an arbitrarily small $\epsilon$ if the algorithm is run for a sufficiently large number of steps. In particular, we can use Corollary~\ref{cor:resilience} to determine, for any $f < n/2$ and $\epsilon > 0$, the number of steps $T$ and the momentum coefficient $\beta$ for which Algorithm~\ref{algo} is $(f, \epsilon)$-resilient, for any of the six aggregation rules listed in Table~\ref{tab:resilience_coefficient}. This shows that, by Definition~\ref{def:resilience}, Algorithm~\ref{algo} is \emph{optimally resilient} for any of these rules. In Table~\ref{tab:epsilon}, we summarize the rates of convergence (i.e., order of $\epsilon$) for the aggregation rules we consider. These rates are simply computed by substituting in Corollary~\ref{cor:resilience} the values of $\lambda$ from Table~\ref{tab:resilience_coefficient}.

\subsection{Analysis \& Discussion}\label{sec:Overview}

\textbf{Impact of the fraction of Byzantine workers.} From Table~\ref{tab:epsilon} we note that the order of $\epsilon$ grows proportionally to $f/n$ for all the aggregation rules listed, except for CWMed. Basically, a smaller fraction of Byzantine workers enables faster convergence to Algorithm~\ref{algo} when using an appropriate resilience averaging rule. 


\textbf{Comparison of convergence rates.}
The rate of convergence of Algorithm~\ref{algo}, shown in Corollary~\ref{cor:resilience}, matches that of vanilla distributed SGD~\cite{lei2019stochastic} in terms of the total number of steps $T$.\footnote{Vanilla distributed SGD refers to the case when the server uses the simple averaging rule and there are no Byzantine workers.} 
Moreover, when the Byzantine workers are very few, i.e., $f \ll n$, the rate for MDA, CWTM, and MeaMed is $\mathcal{O}\left( \nicefrac{\sigma}{\sqrt{n T}} \right)$. Thus, their rate improves with larger $n$ in a similar manner as
vanilla distributed SGD~\cite{lian2015asynchronous}. However, in the same scenario, the rate for Krum$^*$, GM and CWMed is $\mathcal{O}\left( \nicefrac{\sigma \sqrt{n}}{\sqrt{T}} \right)$, i.e., it is directly proportional to $n$. 

This phenomenon could be explained by the fact that Krum$^*$, CWMed, and GM are simply {\em median-based} aggregation rules, without any averaging operation. Thus, the variance of their outputs grows with $n$, as suggested by the standard bounds from {\em order statistics}~\cite{arnold1979bounds, bertsimas2006tight}. On the contrary, MDA, CWTM, and MeaMed perform an averaging operation after filtering out dubious vectors, thus mimicking the variance reduction property of the averaging scheme traditionally used in the vanilla distributed SGD.

\begin{figure*}[!ht]
    \centering
    \includegraphics[width=\textwidth]{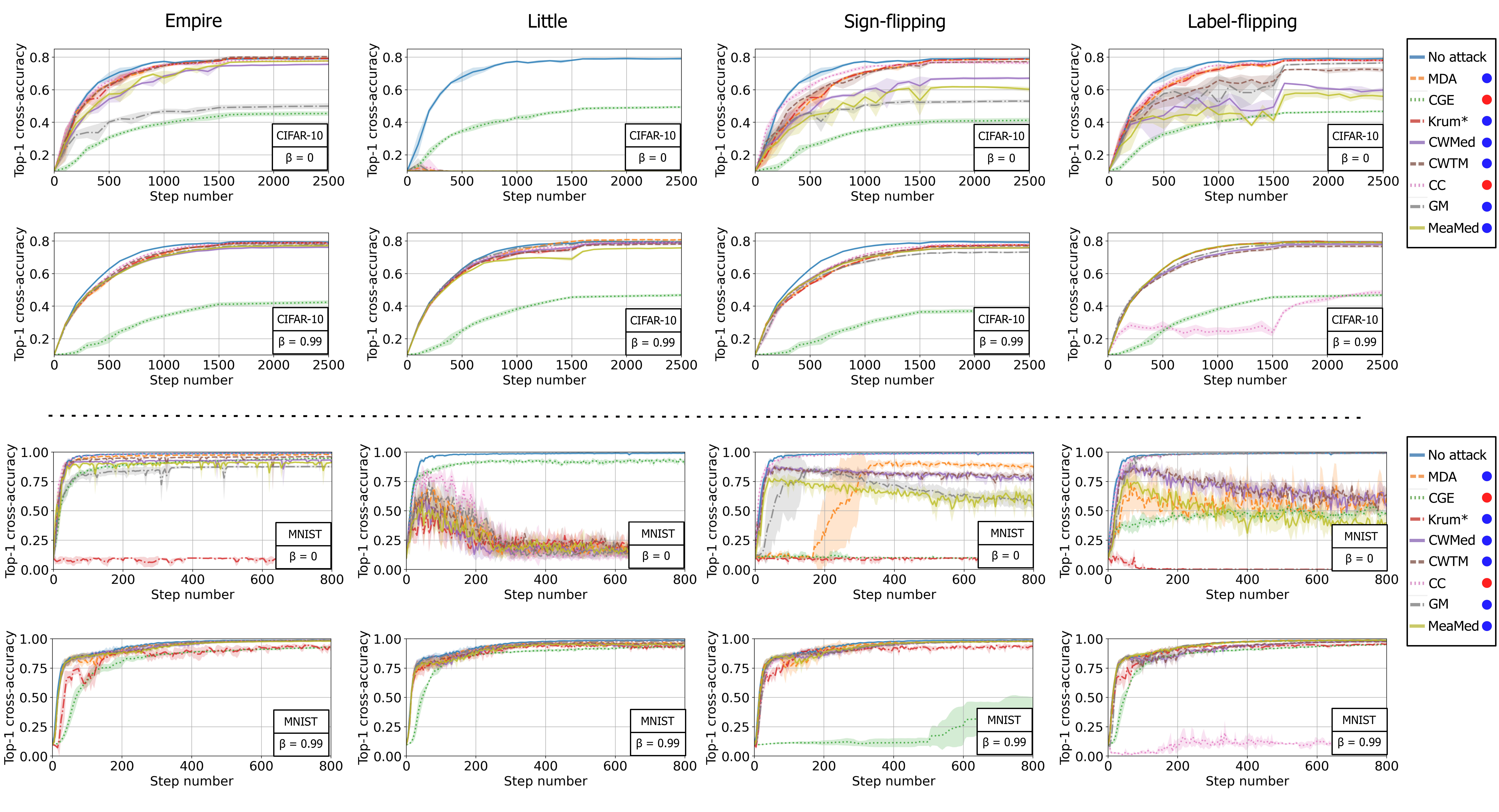}
    \caption{The $1$st and $2$nd rows present experiments performed on CIFAR-10 with $\beta = 0$ and $\beta = 0.99$, respectively. The $3$rd and $4$th rows depict results on MNIST for $\beta = 0$ and $\beta = 0.99$, respectively. The columns convey the performance of the learning under the \textit{empire}, \textit{little}, \textit{sign-flipping}, and \textit{label-flipping} attacks with $f = 5$ Byzantine workers. The resilient averaging rules are marked with a blue circle, while the other {\em non} resilient averaging rules are marked with a red one.}
    \label{fig:plots}
\end{figure*}


\section{Empirical Evaluation}\label{sec:experiments}
To investigate the practical relevance of RESAM, we report on a comprehensive set of experiments evaluating it on benchmark image classification tasks under four different Byzantine threats. We implement Algorithm~\ref{algo} with six different resilient averaging rules and six momentum coefficients. To verify the benefits of our framework, we also run the same set of experiments using two \emph{non} resilient averaging rules. Essentially, our experiments suggest that combining resilient averaging and distributed momentum is critical to Byzantine resilience even in practice.


\subsection{Experimental Setup}\label{expSetup}
\textbf{Datasets.} We use MNIST~\cite{mnist}, Fashion-MNIST~\cite{fashion-mnist}, and CIFAR-10~\cite{cifar}. The datasets are pre-processed as in~\cite{little} and \cite{distributed-momentum}.

\textbf{Architectures and fixed hyperparameters.} For MNIST and Fashion-MNIST, we consider a convolutional neural network (CNN) with two convolutional layers followed by two fully-connected layers. To train the model, we use a Cross Entropy loss, a total number of workers $n = 15$, a constant learning rate $\gamma = 0.75$, and a clipping parameter $C = 2$. We also add an $\ell_2$-regularization factor of $10^{-4}$. Finally, we use a mini-batch size of $b = 25$. For CIFAR-10, we use a CNN with 4 convolutional layers and 2 fully-connected layers, a Cross Entropy loss, and an $\ell_2$-regularization factor of $10^{-2}$. We set $n = 25$, $\gamma = 0.25$, $C = 5$, and $b = 50$. Refer to Appendix~\ref{app:model_arch} for more details on our models.

\textbf{Varying hyperparameters.}
We vary the number of Byzantine workers $f$ in $\left\lbrace 1,3,5,6,7 \right\rbrace$ for MNIST and Fashion-MNIST, and $\left\lbrace 5,11 \right\rbrace$ for CIFAR-10. We also vary the attack implemented by the Byzantine workers. Specifically, we consider \textit{little}~\cite{little}, \textit{empire}~\cite{empire}, \textit{sign-flipping}~\cite{allen2020byzantine}, and \textit{label-flipping}~\cite{allen2020byzantine}. We consider six resilient aggregation rules (MDA, CWTM, CWMed, Krum$^*$, MeaMed, and GM), and two that are not resilient averaging (CGE and CC). As benchmark, we also use the \emph{averaging} aggregation rule without Byzantine workers (denoted by ``No attack''). Finally, we vary the momentum coefficient $\beta$ in $\{0, 0.6, 0.8, 0.9, 0.99, 0.999 \}$.

\textbf{Intractability of MDA and GM.}
Although MDA presents an \textit{order-optimal} resilience coefficient, it is computationally demanding. As pointed out in~\cite{el2020genuinely}, its time complexity is in $\mathcal{O}\left(\binom{n}{f} + d n^2 \right)$. Additionally, GM does not have a closed-form solution. Existing methods implementing GM, such as~\cite{cohen2016geometric,pillutla2019robust} and references therein, are iterative and only approximate GM. Moreover, these methods require expensive computations, e.g., determining eigenvalues and eigenvectors of $d \times d$ matrices~\cite{cohen2016geometric} in each iteration. Here, we use the approximation algorithm from~\cite{pillutla2019robust} to compute GM and only implement MDA whenever its computational complexity is not prohibitive, i.e., when neither $\binom{n}{f}$ nor $d n^2$ are too large. 

\noindent \noindent \fcolorbox{black}{gainsboro!10}{
\parbox{0.45\textwidth}{
\textbf{Reproducibility and reusability.} 
Each experiment is repeated 5 times using seeds from 1 to 5 for reproducibility purposes. Overall, we performed over $1,512$ experiments ($7,560$ runs), of which we provide a brief overview below. Additional plots and code base to reproduce our experiments are available in the supplementary material. Our implementation will also be made accessible online. 
}}

\subsection{Experimental Results}\label{sec:experiment-results}
We present in Figure~\ref{fig:plots} the top-1 cross-accuracy achieved on MNIST and CIFAR-10 when running distributed SGD for 800 and 2500 steps respectively for different aggregation rules and Byzantine attacks. We consider $f = 5$ Byzantine workers in both cases. Due to space limitations, we only show here the results for MNIST and CIFAR-10. Similar results for Fashion-MNIST are deferred to Appendix~\ref{app:exp_results_FM}.

The \textbf{main takeaway} of our experiments is that RESAM is crucial to Byzantine resilience in practice. For all datasets considered, we observe from Figure~\ref{fig:plots} that combining resilient averaging rules (identified by blue points) and distributed momentum (with $\beta = 0.99$) consistently provides similar cross-accuracies as the benchmark (``No attack'') in all attack scenarios. However, when using a resilient averaging rule without momentum ($\beta = 0$), the Byzantine workers can deteriorate the learning (e.g., see second column, \emph{little} attack). Furthermore, using momentum by itself might not suffice either. For instance, on CIFAR-10, using CGE (which is not resilient averaging) results in equally-bad cross-accuracies both when $\beta = 0$ and when $\beta = 0.99$.

\textbf{The case of CC.} In Figure~\ref{fig:plots}, we observe that CC does not present a consistent behavior regarding momentum. In fact, setting $\beta = 0.99$ clearly mitigates the impact of the {\it little} attack, but drastically deteriorates the performance of the algorithm against {\it label-flipping}. Similar inconsistencies are observed for MNIST. Note however that although CC does not behave as a resilient averaging rule, it can present good performances when combined with other levels of momentum (e.g., see $\beta =0.9$ in Appendix~\ref{app:exp_results_CC}).
\vspace{-0.2cm}
\section{Additional Related Work \& Discussion}\label{sec:concusion}

We discuss hereafter other work that we believe to be related to ours, as well as some possible extensions of our approach. 

\textbf{Applicability to robust estimation.} The problem of robust estimation with corrupted data~\cite{lai2016agnostic, charikar2017learning, diakonikolas2017statistical, diakonikolas2019robust, diakonikolas2019sever, steinhardt2018resilience} can be treated as a special case of Byzantine resilience in distributed machine learning where a Byzantine worker behaves just like an honest worker, except that its stochastic gradients may correspond to an incorrect data distribution (instead of $\mathcal{D}$).
RESAM can thus be readily used for robust estimation over an arbitrary distribution $\mathcal{D}$. 

\textbf{Momentum variants.} Besides Polyak's momentum, which we considered, it would be interesting to study the impact of the recently proposed {\em momentum-based variance reduction} (MVR) technique, which has been shown to have optimal convergence rate in non-convex learning~\cite{cutkosky2019momentum}. However, to apply this technique, the gradients (of honest workers) must be defined in a different way than in~\eqref{eqn:grad_i}. Basically, $U(\theta)$ cannot have an arbitrary distribution subject to Assumption~\ref{asp:bnd_var} anymore. 

\textbf{Second-order stationarity.} Although a critical point, i.e., a first-order stationary point, represents a global minimum when the loss function $Q$ is convex, this need not be true in general. Indeed, a critical point may not even represent a local minimum when $Q$ is non-convex, and theoretically speaking, our algorithm may get entrapped at {\em saddle points}. Thus, a stronger learning goal would be to output a second-order stationary point, assuming $Q$ to be second-order Lipschitz smooth. Previous works achieving this goal in the presence of Byzantine workers include~\cite{allen2020byzantine,yin2019defending}. However, they again resort to non-standard assumptions for stochastic gradients. Showing second-order convergence via RESAM under only standard assumptions represents an interesting future work.

\textbf{Non-identical workers.}
When honest workers do not have identical 
data distributions, Byzantine resilience becomes much more challenging~\cite{su2019finite, gupta2020fault,data2021byzantine_icml}. In this case, the goal changes to minimizing the average of the honest workers' loss functions~\cite{su2016fault}. 
More importantly, we cannot achieve a desirable level of resilience anymore unless there is some redundancy in the data~\cite{liu2021approximate}. Apart from using a robust aggregation rule, there has been some work on the use of $\ell_1$-norm regularization~\cite{li2019rsa}. 
Recently,~\cite{karimireddy2020byzantine} also proposed a meta scheme called {\em bucketing} that helps in this setting. 
Extending RESAM to incorporate non-identical honest workers is an interesting future direction.

\textbf{Knowledgeable server.} There is some work studying Byzantine resilience in "non-standard" distributed learning settings where the server either has prior knowledge on specific verified datapoints~\cite{, cao2019distributed, yao2019federated, suspicion, xie2020zeno++, regatti2020bygars}, or has control over the sampling of datapoints~\cite{redundancy, rajput2019detox, gupta2019randomized, data2020data}. In the latter case, we can simply use {\em error-correction coding}. In the former case, we can also tolerate a majority of Byzantine workers. While these solutions might reveal impractical, deriving an optimal condition to overcome the limit of $1/2$ Byzantine workers remains an interesting future direction.

\vspace{-0.2cm}
\section*{Acknowledgments}
Sadegh and Nirupam are partly supported by Swiss National Science Foundation (SNSF) project 200021\_200477, controlling the spread of Epidemics. John is partly supported by SNSF project 200021\_182542, machine learning. Rafaël is partly supported by an Ecocloud postdoctoral fellowship. The authors are thankful to Pierre-Louis Roman for fruitful discussion on the introduction, to Youssef Alouah for proof-reading the technical part, and to the anonymous reviewers of ICML 2022 for their constructive comments. 

\bibliography{example_paper}
\bibliographystyle{icml2022}

\newpage

\appendix
\onecolumn
\addcontentsline{toc}{section}{Appendix} 

\includepdf{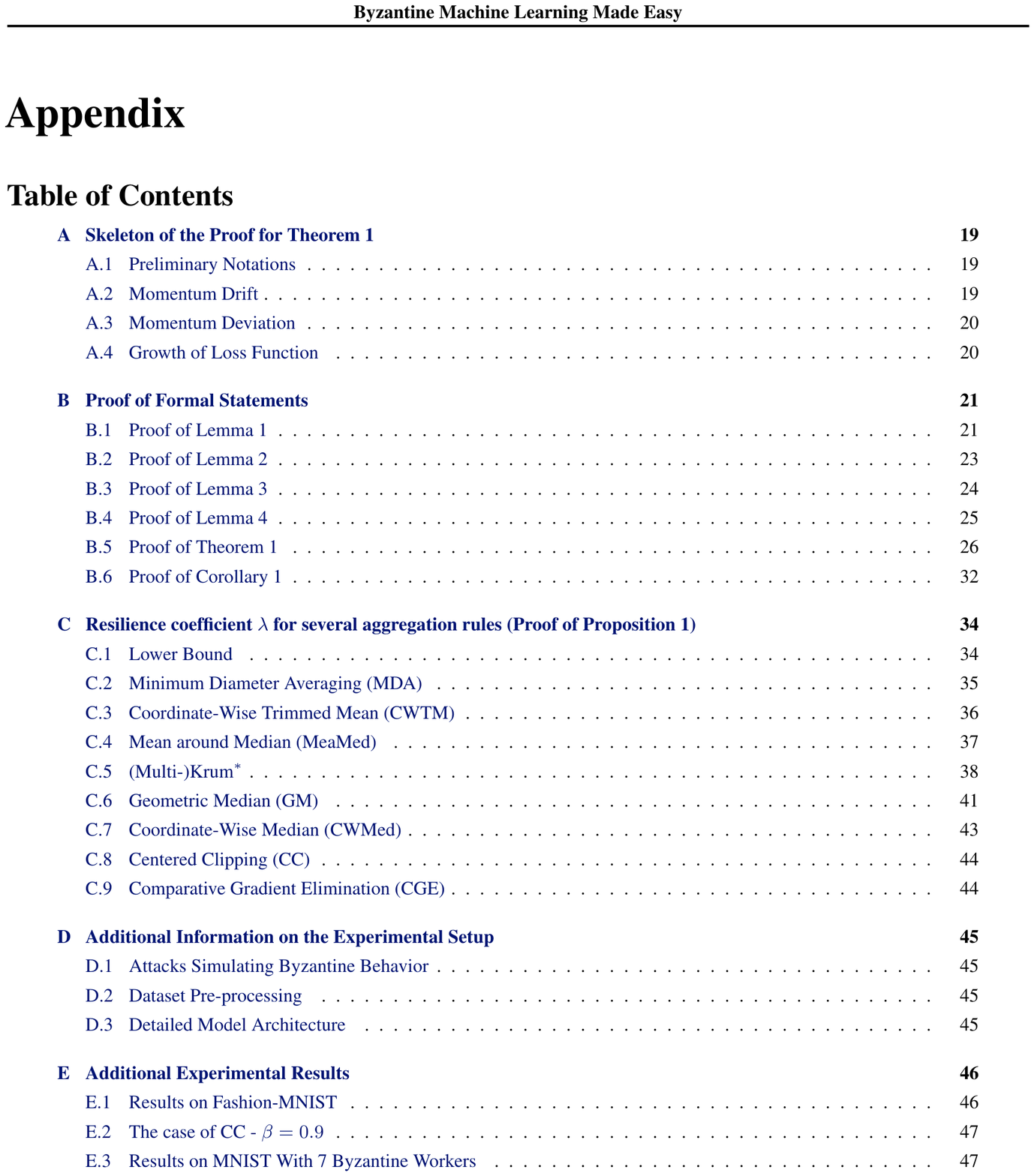}



\newpage
\section{Skeleton of the Proof for Theorem~\ref{thm:main_conv}}\label{sec:resultsskeleton}
Our formal analysis of Algorithm~\ref{algo} constitutes of three critical elements

\begin{enumerate}
    \item The {\em momentum drift} (see Section~\ref{sec:drift})
    \item The {\em momentum deviation} (see Section~\ref{sec:deviation})
    \item The  {\em growth of loss function $Q$} (see Section~\ref{sec:growth})
\end{enumerate}

Ultimately, we combine these elements to obtain the final convergence result stated in Theorem~\ref{thm:main_conv}. Essentially, the proof of Theorem~\ref{thm:main_conv}, deferred to Appendix~\ref{app:main_conv}, is obtained by combining the three sub-results presented by lemmas~\ref{lem:drift},~\ref{lem:dev} and~\ref{lem:growth_Q} below.

\subsection{Preliminary Notations}

For a positive integer $T$, we let $[T]$ denote the set $\{1, \ldots, \, T\}$. For a finite set $S$, we let $\mnorm{S}$ denote its cardinality. For each step $t$, we denote by $R_t$ the output of aggregation rule $F$, i.e., 
\begin{align}
    R_t \coloneqq F\left(\mmt{1}{t}, \ldots, \, \mmt{n}{t} \right). \label{eqn:R}
\end{align}
We denote by $\mathcal{P}_t$ the history from steps $1$ to $t$. Specifically, 
\[\P_t \coloneqq \left\{\weight{1}, \ldots, \, \weight{t}; ~ \mmt{i}{1}, \ldots, \, \mmt{i}{t-1}; i = 1, \ldots, \, n \right\}.\] 
By convention, $\P_1 = \{ \weight{1}\}$. We denote by $\condexpect{t}{\cdot}$ and $\expect{\cdot}$ the conditional expectation $\expect{\cdot ~ \vline ~ \P_t}$ and the total expectation, respectively. Thus, $\expect{\cdot} = \condexpect{1}{ \cdots \condexpect{T}{\cdot}}$.


\subsection{Momentum Drift}
\label{sec:drift}

We first note that at any step $t$, given the history $\P_t$, the momentums $\mmt{i}{t}$ of the honest workers need not be identically distributed, even when the said property is true for their stochastic gradients $\gradient{i}{t}$. Nevertheless, we show in Lemma~\ref{lem:mmt_drift} below that the {\em drift} between the honest workers' momentums can be controlled up to a certain extent by tuning the momentum coefficient $\beta$. We consider an arbitrary subset $\H \subseteq [n]$ of $n-f$ honest workers, i.e., $\mnorm{\H} = n-f$ and $i \in \H$ only if $\worker{}{i}$ is an honest worker. Such a set always exists as there are at least $n-f$ honest workers in the system. Then, defining
\begin{align}
    \AvgMmt{t} \coloneqq \nicefrac{1}{(n-f)} \sum_{i \in \H} \mmt{i}{t}, \label{eqn:def_avg_mmt}
\end{align}
we can demonstrate the following. (Proof of Lemma~\ref{lem:mmt_drift} can be found in Appendix~\ref{app:lem_mmt_drift}.)

\begin{lemma}
\label{lem:mmt_drift}
Suppose that Assumption~\ref{asp:bnd_var} holds true. Consider Algorithm~\ref{algo}. For each $i \in \H$ and $t \in [T]$, we obtain that
\begin{align*}
    \expect{\norm{\mmt{i}{t} - \AvgMmt{t}}^2} \leq 2 \sigma^2 \, (1 - \beta)^2 \beta^{2(t - 1)}  + 2 \left(\frac{1 - \beta}{1 + \beta}\right)  \left(1 + \frac{1}{n-f} \right) \sigma^2. 
\end{align*}
\end{lemma}

Not that the above result holds even when $F$ is not a resilient averaging rule, as it only analyzes the behavior of the worker's momentum. By building upon this first lemma, we can obtain a bound on the distance between the actual output of $F$ and the average momentum of honest workers for the case when $F$ is $(f, \, \lambda)$-resilient averaging. Specifically, when defining
\begin{align}
    \drift{t} \coloneqq R_t - \AvgMmt{t}, \label{eqn:drift}
\end{align}
we get the following. (Proof of Lemma~\ref{lem:drift} can be found in Appendix~\ref{app:lem_drift}.)
\begin{lemma}
\label{lem:drift}
Suppose that Assumption~\ref{asp:bnd_var} holds true. Consider Algorithm~\ref{algo} when $F$ is $(f, \, \lambda)$-resilient averaging. For each step $t \in [T]$, we obtain that
\begin{align*}
    \expect{\norm{\drift{t}}^2} \leq 8 \sigma^2 \lambda^2 (n-f) (1 - \beta)^2 \beta^{2(t - 1)} +  8 \left(\frac{1 - \beta}{1 + \beta}\right) \, \left(n - f + 1 \right) \lambda^2  \, \sigma^2.
\end{align*}
\end{lemma}
\subsection{Momentum Deviation}
\label{sec:deviation}

Next, we study the distance between the average honest momentum $\AvgMmt{t}$ and the true gradient $\nabla Q(\weight{t})$. Specifically, we define {\em deviation} to be
\begin{align}
    \dev{t} \coloneqq \AvgMmt{t} - \nabla Q\left( \weight{t} \right), \label{eqn:dev}
\end{align}
and obtain in Lemma~\ref{lem:dev} below an upper bound on the growth of the deviation over the learning steps $t \in [T]$. (Proof of Lemma~\ref{lem:dev} can be found in Appendix~\ref{app:lem_dev}.) 

\begin{lemma}
\label{lem:dev}
Suppose that assumptions~\ref{asp:lip} and~\ref{asp:bnd_var} hold true. Consider Algorithm~\ref{algo} with $T > 1$. For all $t > 1$ we obtain that
\begin{align*}
    \expect{\norm{\dev{t}}^2} \leq & \beta^2 \zeta_{t-1} \expect{\norm{\dev{t-1}}^2} +  4 \gamma_{t-1}L ( 1 + \gamma_{t-1}L) \beta^2  \expect{\norm{\nabla Q(\weight{t-1})}^2} +(1 - \beta)^2 \frac{\sigma^2}{(n-f)} \\
    & + 2 \gamma_{t-1}L ( 1 + \gamma_{t-1}L)\beta^2  \expect{\norm{\drift{t-1}}^2}.
\end{align*}
where $\zeta_{t} \coloneqq (1 + \gamma_{t}L ) \left(1 + 4 \gamma_{t}  L \right)$.
\end{lemma}

\subsection{Growth of Loss Function} 
\label{sec:growth}

Finally, we analyze the third element, i.e., the growth of cost function $Q(\weight{})$ along the trajectory of Algorithm~\ref{algo}. From~\eqref{eqn:SGD} and~\eqref{eqn:R}, we obtain that for each step $t$
\begin{align*}
    \weight{t+1} = \weight{t} - \gamma_t R_t = \weight{t} - \gamma_t   \, \AvgMmt{t} - \gamma_t \left(R_t -   \, \AvgMmt{t} \right).
\end{align*}
Furthermore, by~\eqref{eqn:drift}, $R_t -   \, \AvgMmt{t} = \drift{t}$. Thus, for all $t$,
\begin{align}
    \weight{t+1} = \weight{t} - \gamma_t   \, \AvgMmt{t} - \gamma_t \drift{t}. \label{eqn:sgd_new}
\end{align}
This means that Algorithm~\ref{algo} can actually be treated as distributed SGD with a momentum term that is subject to perturbation proportional to $\drift{t}$ at each step $t$. This perspective leads us to the following result. (Proof of Lemma~\ref{lem:growth_Q} can be found in Appendix~\ref{app:growth_Q}.) 

\begin{lemma} 
\label{lem:growth_Q}
Suppose that Assumption~\ref{asp:lip} holds true. Consider Algorithm~\ref{algo}. For all $t \in [T]$, we obtain that
\begin{align*}
    \expect{2 Q(\weight{t+1}) - 2 Q(\weight{t})} \leq & - \gamma_t   \left( 1 - 4 \gamma_t L  \right) \expect{\norm{\nabla Q(\weight{t})}^2}  + 2 \gamma_t   \left( 1 + 2 \gamma_t L   \right) \expect{ \norm{\dev{t}}^2} \\
    & + 2 \gamma_t \left(  1 + \gamma_t L \right) \expect{\norm{\drift{t}}^2}.
\end{align*}
\end{lemma}

\newpage

\section{Proof of Formal Statements}
\label{app:proofs}
We now present technical proof for both the aforementioned Lemmas as well as Theorem~\ref{thm:main_conv} and Corollary~\ref{cor:resilience}.

\subsection{Proof of Lemma~\ref{lem:mmt_drift}}
\label{app:lem_mmt_drift}

\noindent \fcolorbox{black}{gainsboro!40}{
\parbox{0.97\textwidth}{\centering
\begin{lemma*}
Suppose that Assumption~\ref{asp:bnd_var} holds true. Consider Algorithm~\ref{algo}. For each $i \in \H$ and $t \in [T]$, we obtain that
\begin{align*}
    \expect{\norm{\mmt{i}{t} - \AvgMmt{t}}^2} \leq 2 \sigma^2 \, (1 - \beta)^2 \beta^{2(t - 1)}  + 2 \left(\frac{1 - \beta}{1 + \beta}\right)  \left(1 + \frac{1}{n-f} \right) \sigma^2. 
\end{align*}
\end{lemma*}
}}
\begin{proof}
Recall that $\H \subseteq \{1, \ldots, \, n\}$ is a set of $n-f$ honest workers, i.e., $\mnorm{\H} = n-f$ and $i \in \H$ only if $\worker{}{i}$ is an honest worker. Also, recall from~\eqref{eqn:def_avg_mmt} that 
\begin{align*}
    \AvgMmt{t} \coloneqq \nicefrac{1}{(n-f)} \sum_{i \in \H} \mmt{i}{t}. 
\end{align*}
We consider an arbitrary $i \in \H$. For simplicity we define $$\diffmmt{i}{t} \coloneqq \mmt{i}{t} - \AvgMmt{t},$$ and
\begin{align}
    \overline{g}_t \coloneqq \nicefrac{1}{(n-f)} \sum_{j \in \H}\gradient{j}{t}. \label{eqn:over_g}
\end{align}
Now, we consider an arbitrary step $t \in [T]$. Substituting from~\eqref{eqn:mmt_i}, i.e., $\mmt{i}{t} = \beta \, \mmt{i}{t-1} + (1 - \beta) \, \gradient{i}{t}$ for all $i \in \H$, in~\eqref{eqn:def_avg_mmt}, i.e., $\AvgMmt{t} = \nicefrac{1}{(n-f)} \sum_{i \in \H} \mmt{i}{t}$, we obtain that
\begin{align*}
    \AvgMmt{t} = \beta \, \AvgMmt{t-1} + (1 - \beta) \, \overline{g}_t
\end{align*}
where $\AvgMmt{0} = 0$, as $\mmt{i}{0} = 0$ for all honest $\worker{}{i}$ by convention. Thus,
\begin{align}
    \diffmmt{i}{t} = \beta \, \diffmmt{i}{t-1} + (1 - \beta) \left( \gradient{i}{t} - \overline{g}_t\right). \label{eqn:tilde_mmt}
\end{align}
Recall that for any vector $v$, $\norm{v}^2 = \iprod{v}{v}$. From above we obtain that
\begin{align*}
    \norm{\diffmmt{i}{t}}^2 = \beta^2 \norm{\diffmmt{i}{t-1}}^2 + (1 - \beta)^2 \norm{\gradient{i}{t} - \overline{g}_t}^2 + 2 \beta (1-\beta)\, \iprod{\diffmmt{i}{t-1}}{\gradient{i}{t} - \overline{g}_t}.
\end{align*}
Upon taking conditional expectation $\condexpect{t}{\cdot}$ on both sides, and using the fact that $\diffmmt{i}{t-1}$ is a deterministic function of the history $\P_t$, we obtain that
\begin{align*}
    \condexpect{t}{\norm{\diffmmt{i}{t}}^2} & = \beta^2 \condexpect{t}{\norm{\diffmmt{i}{t-1}}^2} + (1 - \beta)^2 \condexpect{t}{\norm{\gradient{i}{t} - \overline{g}_t}^2} + 2 \beta (1-\beta)\, \condexpect{t}{\iprod{\diffmmt{i}{t-1}}{\gradient{i}{t} - \overline{g}_t}} \\
    & = \beta^2 \norm{\diffmmt{i}{t-1}}^2 + (1 - \beta)^2 \condexpect{t}{\norm{\gradient{i}{t} - \overline{g}_t}^2} + 2 \beta (1-\beta)\, \iprod{\diffmmt{i}{t-1}}{\condexpect{t}{\gradient{i}{t} - \overline{g}_t}}.
\end{align*}
Due to Assumption~\ref{asp:bnd_var} and the definition of $\overline{g}_t$ in~\eqref{eqn:over_g}, $\condexpect{t}{\gradient{i}{t} - \overline{g}_t} = \condexpect{t}{\gradient{i}{t}} - \condexpect{t}{\overline{g}_t} = \nabla Q(\weight{t}) - \nabla Q(\weight{t}) = 0$. Thus, from above we obtain that
\begin{align*}
    \condexpect{t}{\norm{\diffmmt{i}{t}}^2} = \beta^2 \norm{\diffmmt{i}{t-1}}^2 + (1 - \beta)^2 \condexpect{t}{\norm{\gradient{i}{t} - \overline{g}_t}^2}.
\end{align*}
Assumption~\ref{asp:bnd_var} also implies that $\condexpect{t}{\norm{\gradient{i}{t} - \nabla Q(\weight{t})}^2} \leq \sigma^2$ for all $i \in \H$. As $\gradient{j}{t}$'s for $j \in \H$ are independent of each other, we have $\condexpect{t}{\norm{\overline{g}_t - \nabla Q(\weight{t})}^2} \leq \nicefrac{\sigma^2}{(n-f)}$. Therefore, $\condexpect{t}{\norm{\gradient{i}{t} - \overline{g}_t}^2} \leq 2 \left(1 + \nicefrac{1}{(n-f)} \right) \sigma^2$. Substituting this above we obtain that
\begin{align*}
    \condexpect{t}{\norm{\diffmmt{i}{t}}^2} \leq \beta^2 \norm{\diffmmt{i}{t-1}}^2 + 2 (1 - \beta)^2  \left(1 + \frac{1}{n-f} \right) \sigma^2.
\end{align*}
Taking total expectation on both sides we obtain that
\begin{align*}
    \expect{\norm{\diffmmt{i}{t}}^2} \leq \beta^2 \expect{\norm{\diffmmt{i}{t-1}}^2} + 2 (1 - \beta)^2  \left(1 + \frac{1}{n-f} \right) \sigma^2.
\end{align*}
As the above holds true for an arbitrary $t \in [T]$, 
by telescopic expansion we obtain for all $t \in [T]$ that
\begin{align*}
    \expect{\norm{\diffmmt{i}{t}}^2} & \leq \beta^{2(t - 1)} \expect{\norm{\diffmmt{i}{1}}^2} + 2 (1 - \beta)^2 \left(1 + \frac{1}{n-f} \right) \sigma^2  \sum_{\tau = 0}^{t-2}\beta^{2\tau} \\
    & = \beta^{2(t - 1)} \expect{\norm{\diffmmt{i}{1}}^2} + 2 (1 - \beta)^2 \left(1 + \frac{1}{n-f} \right) \sigma^2 \left(\frac{1 - \beta^{2(t-1)}}{1 - \beta^2} \right).
\end{align*}
As $0 \leq \beta < 1$, we have $1 - \beta^{2(t-1)} \leq 1$. Thus, from above we obtain for all $t \in [T]$ that
\begin{align}
    \expect{\norm{\diffmmt{i}{t}}^2} \leq \beta^{2(t - 1)} \expect{\norm{\diffmmt{i}{1}}^2} + 2 \left(\frac{1 - \beta}{1 + \beta}\right)  \left(1 + \frac{1}{n-f} \right) \sigma^2.  \label{eqn:tilde_mmt_telescope}
\end{align}
From~\eqref{eqn:tilde_mmt}, for each $i \in \H$ we have (upon recalling that $\mmt{i}{0} = 0$ for all $i \in \H$),
\begin{align*}
    \diffmmt{i}{1} = ( 1 - \beta) \left( \gradient{i}{1} - \overline{g}_1\right).
\end{align*}
By definition of $\overline{g}_t$ in~\eqref{eqn:over_g}, 
\begin{align*}
    \expect{\norm{\diffmmt{i}{1}}^2} = ( 1 - \beta)^2 \expect{ \norm{\gradient{i}{1} - \overline{g}_1}^2} = ( 1 - \beta)^2 \expect{ \norm{\frac{1}{(n-f)} \sum_{j \in \H}\left(\gradient{i}{1} - \gradient{j}{1}\right)}^2}.
\end{align*}
Thus, by applying Jensen's inequality, 
\begin{align*}
    \expect{\norm{\diffmmt{i}{1}}^2} \leq \frac{( 1 - \beta)^2}{(n-f)} \sum_{j \in \H} \expect{\norm{\gradient{i}{1} - \gradient{j}{1}}^2}.
\end{align*}
By Assumption~\ref{asp:bnd_var}, as gradients of honest workers are pair-wise independent, $\expect{\norm{\gradient{i}{1} - \gradient{j}{1}}^2} \leq 2 \sigma^2$. Substituting this above we obtain that for each $i \in \H$,
\begin{align*}
    \expect{\norm{\diffmmt{i}{1}}^2} \leq  2 \sigma^2(1 - \beta)^2.
\end{align*}
Substituting from above in~\eqref{eqn:tilde_mmt_telescope} proves the lemma, i.e., for all $t \in [T]$,
\begin{align*}
    \expect{\norm{\diffmmt{i}{t}}^2} \leq 2 \sigma^2 \, (1 - \beta)^2 \beta^{2(t - 1)}  + 2 \left(\frac{1 - \beta}{1 + \beta}\right)  \left(1 + \frac{1}{n-f} \right) \sigma^2. 
\end{align*}
\end{proof}

\subsection{Proof of Lemma~\ref{lem:drift}}
\label{app:lem_drift}
 
 \noindent \fcolorbox{black}{gainsboro!40}{
\parbox{0.97\textwidth}{\centering
\begin{lemma*}
Suppose that Assumption~\ref{asp:bnd_var} holds true. Consider Algorithm~\ref{algo} when $F$ is $(f, \, \lambda)$-resilient averaging. For each step $t \in [T]$, we obtain that
\begin{align*}
    \expect{\norm{\drift{t}}^2} \leq 8 \sigma^2 \lambda^2 (n-f) (1 - \beta)^2 \beta^{2(t - 1)} +  8 \left(\frac{1 - \beta}{1 + \beta}\right) \, \left(n - f + 1 \right) \lambda^2  \, \sigma^2.
\end{align*}
\end{lemma*}
}}

\begin{proof}
Recall from~\eqref{eqn:R} and~\eqref{eqn:drift}, respectively, that
\begin{align*}
    R_t \coloneqq F\left(\mmt{1}{t}, \ldots, \, \mmt{n}{t} \right) ~ \text{ and } ~ \drift{t} \coloneqq R_t - \AvgMmt{t}.
\end{align*}
We consider an arbitrary step $t$. 
As $F$ is assumed $(f,\lambda)$-resilient averaging, by Definition~\ref{def:rational} we obtain that
\begin{align}
    \norm{\drift{t}}^2 = \norm{R_t - \AvgMmt{t}}^2 \leq \lambda^2 \max_{i, \, j \in \H} \norm{\mmt{i}{t} - \mmt{j}{t}}^2.  \label{eqn:lem_drift_t}
\end{align}
Note that for any pair $i, \, j \in \H$, from triangle inequality we have $\norm{\mmt{i}{t} - \mmt{j}{t}} \leq \norm{\mmt{i}{t} - \AvgMmt{t}} + \norm{\mmt{j}{t} - \AvgMmt{t}}$. As $2ab \leq a^2 + b^2$, we also have $\norm{\mmt{i}{t} - \mmt{j}{t}}^2 \leq 2\norm{\mmt{i}{t} - \AvgMmt{t}}^2 + 2\norm{\mmt{j}{t} - \AvgMmt{t}}^2 \leq 4 \max_{i \in \H} \norm{\mmt{i}{t} - \AvgMmt{t}}^2$. Therefore, 
\begin{align*}
    \max_{i, \, j \in \H} \norm{\mmt{i}{t} - \mmt{j}{t}}^2 \leq 4 \max_{i \in \H} \norm{\mmt{i}{t} - \AvgMmt{t}}^2 .
\end{align*}
As $\max_{i \in \H} \norm{\mmt{i}{t} - \AvgMmt{t}}^2 \leq \sum_{i \in \H} \norm{\mmt{i}{t} - \AvgMmt{t}}^2$, from above we obtain that
\begin{align*}
    \max_{i, \, j \in \H} \norm{\mmt{i}{t} - \mmt{j}{t}}^2 \leq 4 \sum_{i \in \H} \norm{\mmt{i}{t} - \AvgMmt{t}}^2.
\end{align*}
Substituting from above in~\eqref{eqn:lem_drift_t} we obtain that $\norm{\drift{t}}^2 \leq 4 \lambda^2 \sum_{i \in \H} \norm{\mmt{i}{t} - \AvgMmt{t}}^2$. Upon taking total expectations on both sides we obtain that
\begin{align}
    \expect{\norm{\drift{t}}^2} \leq 4 \lambda^2 \sum_{i \in \H} \expect{\norm{\mmt{i}{t} - \AvgMmt{t}}^2}. \label{eqn:before_lemma_drift}
\end{align}
From Lemma~\ref{lem:mmt_drift}, under Assumption~\ref{asp:bnd_var}, we have for all $i \in \H$ that 
\begin{align*}
    \expect{\norm{\mmt{i}{t} - \AvgMmt{t}}^2} \leq 2 \sigma^2 \, (1 - \beta)^2 \beta^{2(t - 1)}  + 2 \left(\frac{1 - \beta}{1 + \beta}\right)  \left(1 + \frac{1}{n-f} \right) \sigma^2. 
\end{align*}
As $\mnorm{\H} = n-f$, Substituting from above in~\eqref{eqn:before_lemma_drift} proves the lemma, i.e., we obtain that
\begin{align*}
     \expect{\norm{\drift{t}}^2} \leq 8 \lambda^2 \sigma^2 (n-f) (1 - \beta)^2 \beta^{2(t - 1)} + 8 \lambda^2 \left(\frac{1 - \beta}{1 + \beta}\right) \left( n- f + 1 \right) \sigma^2.
\end{align*}
\end{proof}
\subsection{Proof of Lemma~\ref{lem:dev}}
\label{app:lem_dev}


\noindent \fcolorbox{black}{gainsboro!40}{
\parbox{0.97\textwidth}{\centering
\begin{lemma*}
Suppose that assumptions~\ref{asp:lip} and~\ref{asp:bnd_var} hold true. Consider Algorithm~\ref{algo} with $T > 1$. For all $t > 1$ we obtain that
\begin{align*}
    \expect{\norm{\dev{t}}^2} \leq & \beta^2 \zeta_{t-1} \expect{\norm{\dev{t-1}}^2} +  4 \gamma_{t-1}L ( 1 + \gamma_{t-1}L) \beta^2  \expect{\norm{\nabla Q(\weight{t-1})}^2} +(1 - \beta)^2 \frac{\sigma^2}{(n-f)} \\
    & + 2 \gamma_{t-1}L ( 1 + \gamma_{t-1}L)\beta^2  \expect{\norm{\drift{t-1}}^2}.
\end{align*}
where $\zeta_{t} \coloneqq (1 + \gamma_{t}L ) \left(1 + 4 \gamma_{t}  L \right)$.
\end{lemma*}
}}

\begin{proof}
Recall from Definition~\eqref{eqn:dev} that 
\begin{align*}
    \dev{t} \coloneqq \AvgMmt{t} - \nabla Q\left( \weight{t} \right).
\end{align*}
Consider an arbitrary step $t > 1$. By Definitions~\eqref{eqn:mmt_i} and~\eqref{eqn:def_avg_mmt}, we obtain that
\begin{align*}
    \dev{t} = \beta \, \AvgMmt{t-1} + (1 - \beta) \, \overline{g}_{t} - \nabla Q \left( \weight{t} \right).
\end{align*}
Upon adding and subtracting $\beta \nabla Q(\weight{t-1})$ and $\beta \nabla Q(\weight{t})$ on the R.H.S.~above we obtain that
\begin{align*}
    \dev{t} & = \beta \, \AvgMmt{t-1} - \beta \nabla Q(\weight{t-1}) + (1 - \beta) \, \overline{g}_{t} - \nabla Q \left( \weight{t} \right) + \beta \nabla Q(\weight{t}) + \beta \nabla Q(\weight{t-1}) - \beta \nabla Q(\weight{t}) \\
    & = \beta \left( \AvgMmt{t-1} - \nabla Q(\weight{t-1}) \right) + (1 - \beta) \, \overline{g}_{t} - (1 - \beta) \nabla Q \left( \weight{t} \right) + \beta \left( \nabla Q(\weight{t-1}) - \nabla Q(\weight{t})  \right).
\end{align*}
As $\AvgMmt{t-1} - \nabla Q(\weight{t-1}) = \dev{t-1}$ (by Definition~\eqref{eqn:dev}), from above we obtain that
\begin{align*}
    \dev{t} = \beta \dev{t-1} + (1 - \beta) \, \left( \overline{g}_{t} - \nabla Q \left( \weight{t} \right) \right) + \beta \left( \nabla Q(\weight{t-1}) - \nabla Q(\weight{t}) \right).
\end{align*}
Therefore, 
\begin{align*}
    \norm{\dev{t}}^2 = & \beta^2 \norm{\dev{t-1}}^2 + (1 - \beta)^2 \norm{ \overline{g}_{t} - \nabla Q \left( \weight{t} \right)}^2 + \beta^2 \norm{\nabla Q(\weight{t-1}) - \nabla Q(\weight{t}) }^2 + 2 \beta (1 - \beta) \iprod{\dev{t-1}}{\overline{g}_{t} - \nabla Q \left( \weight{t} \right)} \\
    & + 2 \beta^2 \iprod{\dev{t-1}}{\nabla Q(\weight{t-1}) - \nabla Q(\weight{t})} + 2 \beta ( 1- \beta) \iprod{\overline{g}_{t} - \nabla Q \left( \weight{t} \right)}{\nabla Q(\weight{t-1}) - \nabla Q(\weight{t})}.
\end{align*}
By taking conditional expectation $\condexpect{t}{\cdot}$ on both sides, and recalling that $\dev{t-1}$, $\weight{t}$ and $\weight{t-1}$ are deterministic values when the history $\P_t$ is given, we obtain that 
\begin{align*}
    \condexpect{t}{\norm{\dev{t}}^2} = & \beta^2 \norm{\dev{t-1}}^2 + (1 - \beta)^2 \condexpect{t}{\norm{ \overline{g}_{t} - \nabla Q \left( \weight{t} \right)}^2} + \beta^2 \norm{\nabla Q(\weight{t-1}) - \nabla Q(\weight{t}) }^2 + 2 \beta (1 - \beta) \iprod{\dev{t-1}}{\condexpect{t}{\overline{g}_{t}} - \nabla Q \left( \weight{t} \right)} \\
    & + 2 \beta^2 \iprod{\dev{t-1}}{\nabla Q(\weight{t-1}) - \nabla Q(\weight{t})} + 2 \beta ( 1- \beta) \iprod{\condexpect{t}{\overline{g}_{t}} - \nabla Q \left( \weight{t} \right)}{\nabla Q(\weight{t-1}) - \nabla Q(\weight{t})}.
\end{align*}
Recall that $\overline{g}_t \coloneqq \nicefrac{1}{(n-f)} \sum_{j \in \H}\gradient{j}{t}$. Thus, owing to Assumption~\ref{asp:bnd_var}, $\condexpect{t}{\overline{g}_{t}} = \nabla Q(\weight{t})$. Using this above we obtain that
\begin{align*}
    \condexpect{t}{\norm{\dev{t}}^2} = & \beta^2 \norm{\dev{t-1}}^2 + (1 - \beta)^2 \condexpect{t}{\norm{ \overline{g}_{t} - \nabla Q \left( \weight{t} \right)}^2} + \beta^2 \norm{\nabla Q(\weight{t-1}) - \nabla Q(\weight{t}) }^2 \\
    & + 2 \beta^2 \iprod{\dev{t-1}}{\nabla Q(\weight{t-1}) - \nabla Q(\weight{t})}.
\end{align*}
Also, by Assumption~\ref{asp:bnd_var} and the fact that $\gradient{j}{t}$'s for $j \in \H$ are independent of each other, we have $\condexpect{t}{\norm{ \overline{g}_{t} - \nabla Q \left( \weight{t} \right)}^2} \leq \frac{\sigma^2}{(n-f)}$. Thus,
\begin{align*}
    \condexpect{t}{\norm{\dev{t}}^2} \leq \beta^2 \norm{\dev{t-1}}^2 + (1 - \beta)^2 \frac{\sigma^2}{(n-f)} + \beta^2 \norm{\nabla Q(\weight{t-1}) - \nabla Q(\weight{t}) }^2 + 2 \beta^2 \iprod{\dev{t-1}}{\nabla Q(\weight{t-1}) - \nabla Q(\weight{t})}.
\end{align*}
By Cauchy-Schwartz inequality, $\iprod{\dev{t-1}}{\nabla Q(\weight{t-1}) - \nabla Q(\weight{t})} \leq \norm{\dev{t-1}} \norm{\nabla Q(\weight{t-1}) - \nabla Q(\weight{t})}$. By Assumption~\ref{asp:lip}, $\norm{ \nabla Q(\weight{t-1}) - \nabla Q(\weight{t})} \leq L \norm{\weight{t} - \weight{t-1}}$. Recall from~\eqref{eqn:SGD} that $\weight{t} = \weight{t-1} - \gamma_{t-1} R_{t-1}$. Thus,\\ $\norm{\nabla Q(\weight{t-1}) - \nabla Q(\weight{t})} \leq \gamma_{t-1} L \norm{R_{t-1}}$. Using this above we obtain that
\begin{align*}
    \condexpect{t}{\norm{\dev{t}}^2} \leq \beta^2 \norm{\dev{t-1}}^2 + (1 - \beta)^2 \frac{\sigma^2}{(n-f)} + \gamma^2_{t-1} \beta^2 L^2 \norm{R_{t-1}}^2 + 2 \gamma_{t-1} \beta^2 L \norm{\dev{t-1}} \norm{R_{t-1}}.
\end{align*}
As $2 ab \leq a^2 + b^2$, from above we obtain that
\begin{align}
    \condexpect{t}{\norm{\dev{t}}^2} & \leq \beta^2 \norm{\dev{t-1}}^2 + (1 - \beta)^2 \frac{\sigma^2}{(n-f)} + \gamma^2_{t-1} \beta^2 L^2 \norm{R_{t-1}}^2 + \gamma_{t-1} L \beta^2 \left( \norm{\dev{t-1}}^2 +  \norm{R_{t-1}}^2\right) \nonumber \\
    & = (1 + \gamma_{t-1}L ) \beta^2 \norm{\dev{t-1}}^2 + (1 - \beta)^2 \frac{\sigma^2}{(n-f)} + \gamma_{t-1}L (1 + \gamma_{t-1}L) \beta^2  \norm{R_{t-1}}^2. \label{eqn:dev_before_ab}
\end{align}
By definition of $R_t$ in~\eqref{eqn:drift}, $R_{t-1} = \drift{t-1} +  \AvgMmt{t-1}$. Thus, owing to the triangle inequality and the fact that $2 ab \leq a^2 + b^2$, we have $\norm{R_{t-1}}^2 \leq 2 \norm{\drift{t-1}}^2 + 2  \norm{\AvgMmt{t-1}}^2$. Similarly, by definition of $\dev{t}$ in~\eqref{eqn:dev}, we have $\norm{\AvgMmt{t-1}}^2 \leq 2 \norm{\dev{t-1}}^2 + 2 \norm{\nabla Q(\weight{t-1})}^2$. Thus, $\norm{R_{t-1}}^2 \leq 2 \norm{\drift{t-1}}^2 + 4  \norm{\dev{t-1}}^2 + 4  \norm{\nabla Q(\weight{t-1})}^2$. 
Using this in~\eqref{eqn:dev_before_ab} we obtain that
\begin{align*}
    \condexpect{t}{\norm{\dev{t}}^2} \leq & (1 + \gamma_{t-1}L ) \beta^2 \norm{\dev{t-1}}^2 + (1 - \beta)^2 \frac{\sigma^2}{(n-f)} \\
    & + 2 \gamma_{t-1}L( 1 + \gamma_{t-1}L)\beta^2  \left( \norm{\drift{t-1}}^2 + 2  \norm{\dev{t-1}}^2 + 2  \norm{\nabla Q(\weight{t-1})}^2 \right).
\end{align*}
By re-arranging the terms on the R.H.S.~we get
\begin{align*}
    \condexpect{t}{\norm{\dev{t}}^2} \leq & \beta^2 (1 + \gamma_{t-1}L ) \left(1 + 4 \gamma_{t-1}  L \right) \norm{\dev{t-1}}^2 +  4 \gamma_{t-1} L( 1 + \gamma_{t-1}L) \beta^2   \norm{\nabla Q(\weight{t-1})}^2 +(1 - \beta)^2 \frac{\sigma^2}{(n-f)} \\
    & + 2 \gamma_{t-1} L( 1 + \gamma_{t-1}L)\beta^2 \norm{\drift{t-1}}^2.
\end{align*}
Substituting $\zeta_{t-1} = (1 + \gamma_{t-1} L) \left(1 + 4 \gamma_{t-1}  L \right)$ above we obtain that
\begin{align*}
    \condexpect{t}{\norm{\dev{t}}^2} \leq & \beta^2 \zeta_{t-1} \norm{\dev{t-1}}^2 +  4 \gamma_{t-1}L ( 1 + \gamma_{t-1}L) \beta^2 \norm{\nabla Q(\weight{t-1})}^2 +(1 - \beta)^2 \frac{\sigma^2}{(n-f)} + 2 \gamma_{t-1}L ( 1 + \gamma_{t-1}L)\beta^2 \norm{\drift{t-1}}^2.
\end{align*}
Recall that $t$ in the above is an arbitrary value in $[T]$ greater than $1$. Hence, upon taking total expectation on both sides above proves the lemma.

\end{proof}

\subsection{Proof of Lemma~\ref{lem:growth_Q}}
\label{app:growth_Q}


\noindent \fcolorbox{black}{gainsboro!40}{
\parbox{0.97\textwidth}{\centering
\begin{lemma*} 
Suppose that Assumption~\ref{asp:lip} holds true. Consider Algorithm~\ref{algo}. For all $t \in [T]$, we obtain that
\begin{align*}
    \expect{2 Q(\weight{t+1}) - 2 Q(\weight{t})} \leq & - \gamma_t   \left( 1 - 4 \gamma_t L  \right) \expect{\norm{\nabla Q(\weight{t})}^2}  + 2 \gamma_t   \left( 1 + 2 \gamma_t L   \right) \expect{ \norm{\dev{t}}^2} \\
    & + 2 \gamma_t \left(  1 + \gamma_t L \right) \expect{\norm{\drift{t}}^2}.
\end{align*}
\end{lemma*}
}}

\begin{proof}
Consider an arbitrary step $t$. Due to Assumption~\ref{asp:lip} (i.e., Lipschitz continuity of $\nabla Q(\weight{})$), we have (see~\cite{bottou2018optimization})
\begin{align*}
    Q(\weight{t+1}) - Q(\weight{t}) \leq \iprod{\weight{t+1} - \weight{t}}{\nabla Q(\weight{t})} + \frac{L}{2} \norm{\weight{t+1} - \weight{t}}^2.
\end{align*}
Substituting from~\eqref{eqn:sgd_new}, i.e., $\weight{t+1} = \weight{t} - \gamma_t   \, \AvgMmt{t} - \gamma_t \drift{t}$, we obtain that
\begin{align*}
    Q(\weight{t+1}) - Q(\weight{t}) &\leq - \gamma_t  \iprod{\AvgMmt{t}}{\nabla Q(\weight{t})} - \gamma_t \iprod{\drift{t}}{\nabla Q(\weight{t})} + \gamma^2_t \frac{L}{2} \norm{ \, \AvgMmt{t} + \drift{t}}^2. \\
    & = - \gamma_t  \iprod{\AvgMmt{t} - \nabla Q(\weight{t}) + \nabla Q(\weight{t})}{\nabla Q(\weight{t})} - \gamma_t \iprod{\drift{t}}{\nabla Q(\weight{t})} + \gamma^2_t \frac{L}{2} \norm{ \, \AvgMmt{t} + \drift{t}}^2.
\end{align*}
By Definition~\eqref{eqn:dev}, $\AvgMmt{t} - \nabla Q(\weight{t}) = \dev{t}$. Thus, from above we obtain that (scaling by factor of $2$)
\begin{align}
    2 Q(\weight{t+1}) - 2 Q(\weight{t}) \leq - 2 \gamma_t  \norm{\nabla Q(\weight{t})}^2 - 2 \gamma_t  \iprod{\dev{t}}{\nabla Q(\weight{t})} - 2 \gamma_t \iprod{\drift{t}}{\nabla Q(\weight{t})} + \gamma^2_t L \norm{ \, \AvgMmt{t} + \drift{t}}^2. \label{eqn:norm_1}
\end{align}
Now, we consider the last three terms on the R.H.S.~separately. Using Cauchy-Schwartz inequality, and the fact that $2 ab \leq \frac{1}{c} a^2 + c b^2$ for any $c > 0$, we obtain that (by substituting $c = 2$)
\begin{align}
    2 \mnorm{\iprod{\dev{t}}{\nabla Q(\weight{t})}} \leq 2 \norm{\dev{t}} \norm{\nabla Q(\weight{t})} \leq \frac{2}{1} \norm{\dev{t}}^2 + \frac{1}{2} \norm{\nabla Q(\weight{t})}^2 . \label{eqn:rho_1}
\end{align}
Similarly, 
\begin{align}
    2 \mnorm{\iprod{\drift{t}}{\nabla Q(\weight{t})}} \leq 2 \norm{\drift{t}} \norm{\nabla Q(\weight{t})} \leq \frac{ 2}{ 1} \norm{\drift{t}}^2 +  \frac{1}{2} \norm{\nabla Q(\weight{t})}^2. \label{eqn:rho_2}
\end{align}
Finally, using triangle inequality and the fact that $2ab \leq a^2 + b^2$ we have
\begin{align}
    \norm{ \, \AvgMmt{t} + \drift{t}}^2 & \leq 2  \, \norm{\AvgMmt{t}}^2 + 2 \norm{\drift{t}}^2 = 2  \, \norm{\AvgMmt{t} - \nabla Q(\weight{t}) + \nabla Q(\weight{t})}^2 + 2 \norm{\drift{t}}^2 \nonumber \\
    & \leq 4  \, \norm{\dev{t}}^2 + 4  \, \norm{\nabla Q(\weight{t})}^2 + 2 \norm{\drift{t}}^2. \quad \quad [since ~ ~ \AvgMmt{t} - \nabla Q(\weight{t}) = \dev{t}] \label{eqn:last_term}
\end{align}
Substituting from~\eqref{eqn:rho_1},~\eqref{eqn:rho_2} and~\eqref{eqn:last_term} in~\eqref{eqn:norm_1} we obtain that
\begin{align*}
    2 Q(\weight{t+1}) - 2 Q(\weight{t}) \leq & - 2 \gamma_t  \norm{\nabla Q(\weight{t})}^2 + \gamma_t  \left( 2 \norm{\dev{t}}^2 + \frac{1}{2}\norm{\nabla Q(\weight{t})}^2 \right) + \gamma_t \left( 2 \norm{\drift{t}}^2 + \frac{1}{2} \norm{\nabla Q(\weight{t})}^2 \right) \nonumber \\
    & + \gamma^2_t L \left( 4  \, \norm{\dev{t}}^2 + 4  \, \norm{\nabla Q(\weight{t})}^2 + 2 \norm{\drift{t}}^2 \right).
\end{align*}
Upon re-arranging the terms in the R.H.S.~we obtain that
\begin{align*}
    2 Q(\weight{t+1}) - 2 Q(\weight{t}) \leq - \gamma_t  \left( 1 - 4 \gamma_t L \right) \norm{\nabla Q(\weight{t})}^2  + 2 \gamma_t  \left( 1 + 2 \gamma_t L  \right) \norm{\dev{t}}^2 + 2 \gamma_t \left(  1 + \gamma_t L \right) \norm{\drift{t}}^2.
\end{align*}
As $t$ is arbitrarily chosen from $[T]$, taking expectation on both sides above proves the lemma.
\end{proof}
\subsection{Proof of Theorem~\ref{thm:main_conv}}
\label{app:main_conv}

We recall the theorem statement below for convenience.

\noindent \fcolorbox{black}{gainsboro!40}{
\parbox{0.98\textwidth}{\centering
\begin{theorem*}
Suppose that assumptions~\ref{asp:lip} and~\ref{asp:bnd_var} hold true. Let us denote
\begin{align}
    Q^* = \min_{\weight{} \in \R^d} Q(\weight{}), ~ a_o = 4 \left(2 \left(Q(\weight{1}) - Q^* \right) + \frac{1}{8L} \norm{\nabla Q \left( \weight{1} \right)}^2 \right), ~ a_1 = 6912L, ~ \text{ and } 
    a_2 =288L. \label{eqn:a1a2}
\end{align}
Consider Algorithm~\ref{algo} with an $(f,\lambda)$-resilient averaging rule and a constant learning rate of $\gamma$. Specifically, for all $t$, $\gamma_t = \gamma$ where
\begin{align}
    \gamma =  \left( \sqrt{\frac{a_o (n-f)}{a_1 \lambda^2 (n-f)^2 + a_2} } \right) \frac{1}{\sigma \sqrt{T}}.  \label{eqn:step-size}
\end{align}
If $T \geq \frac{a_o L }{12 \sigma^2 \lambda^2 (n - f)}$ and $\beta = \sqrt{1 - 24 \gamma L} $, then
\begin{align*}
   \expect{\norm{\nabla Q\left(\widehat{\weight{}} \right)}^2} \leq 2 \sqrt{ \left( a_1 \lambda^2 (n - f) + \frac{ a_2}{n-f}\right)  \frac{a_o \sigma^2}{T} } + \left( \frac{a_2 \sigma}{n-f}\right) \left( \sqrt{\frac{a_o (n-f)}{a_1 \lambda^2 (n-f) + a_2} } \right) \frac{1}{T^{\nicefrac{3}{2}}}.
\end{align*}
\end{theorem*}
}}
\begin{proof}
Define 
\begin{equation}
    \gamma_o := \frac{1}{18L}.
\end{equation}
Note that as specified in the theorem statement,
\begin{align*}
    T \geq \frac{a_o L }{12 \sigma^2 \lambda^2 (n - f)} \geq \frac{576 a_o L^2}{ 12 \times 576  \sigma^2 \lambda^2 (n - f) L} > \frac{576 a_o L^2 (n - f)}{6912  \sigma^2 \lambda^2 (n - f)^2 L + 288L \sigma^2} = \frac{576L^2 a_o (n-f)}{ \left( a_1 \lambda^2 (n-f)^2 + a_2 \right) \sigma^2}
\end{align*}
This implies that for the learning rate $\gamma$ defined in~\eqref{eqn:step-size},
\begin{align}
     \gamma = \left( \sqrt{\frac{a_o(n-f)}{a_1 \lambda^2 (n-f)^2 + a_2} } \right) \frac{1}{\sigma \sqrt{T}} < \frac{1}{24L} < \frac{1}{18L} = \gamma_o.
     \label{eqn:gamma_gamma_o}
\end{align}
This also implies 
\begin{align*}
    24 \gamma L  < 1
\end{align*}
Therefore $\beta = \sqrt{1 - 24  \gamma L}$ (as defined in the theorem) is a well-defined real value in $(0, 1)$.\\

To obtain the convergence result we 
define the Lyapunov function to be
\begin{align}
    V_t \coloneqq \expect{2 Q(\weight{t}) + z \norm{\dev{t}}^2} ~ \text{ and } z = \frac{1}{8L}. \label{eqn:lyap_func}
\end{align}
We consider an arbitrary $t \in [T]$. \\

{\bf Invoking Lemma~\ref{lem:dev}.} Upon substituting $\gamma_t = \gamma$ in Lemma~\ref{lem:dev}, we obtain that 
\begin{align}
    \expect{ z \norm{\dev{t+1}}^2 - z \norm{\dev{t}}^2} \leq & z \beta^2 \zeta \expect{\norm{\dev{t}}^2} +  4 z \gamma L ( 1 + \gamma L) \beta^2   \expect{\norm{\nabla Q(\weight{t})}^2} + z (1 - \beta)^2 \frac{\sigma^2}{n-f} \nonumber \\
    & + 2 z \gamma L ( 1 + \gamma L)\beta^2 \expect{\norm{\drift{t}}^2} - z \expect{\norm{\dev{t}}^2}. \label{eqn:dev_gamma}
\end{align}
Recall that 
\begin{align}
    \zeta = (1 + \gamma L) \left(1 + 4 \gamma  L \right) = 1 + 5 \gamma L + 4 \gamma^2  L^2. \label{eqn:zeta_expand}
\end{align}

{\bf Invoking Lemma~\ref{lem:growth_Q}.} By the same substitution in Lemma~\ref{lem:growth_Q} we obtain that
\begin{align}
    \expect{2 Q(\weight{t+1}) - 2 Q(\weight{t})} \leq & - \gamma  \left( 1 - 4 \gamma L \right) \expect{\norm{\nabla Q(\weight{t})}^2}  + 2 \gamma  \left( 1 + 2 \gamma L  \right) \expect{ \norm{\dev{t}}^2} \nonumber \\
    & + 2 \gamma \left(  1 + \gamma L \right) \expect{\norm{\drift{t}}^2} \label{eqn:growth_gamma}
\end{align}
Substituting from~\eqref{eqn:dev_gamma} and~\eqref{eqn:growth_gamma} in~\eqref{eqn:lyap_func} we obtain that
\begin{align}
    V_{t+1} - V_t = & \expect{2 Q(\weight{t+1}) - 2 Q(\weight{t})} + \expect{ z \norm{\dev{t+1}}^2 - z \norm{\dev{t}}^2} \nonumber \\
    \leq & - \gamma  \left( 1 - 4 \gamma L \right) \expect{\norm{\nabla Q(\weight{t})}^2}  + 2 \gamma  \left( 1 + 2 \gamma L  \right) \expect{ \norm{\dev{t}}^2} + 2 \gamma \left(  1 + \gamma L \right) \expect{\norm{\drift{t}}^2} \nonumber \\
    & +  z \beta^2 \zeta \expect{\norm{\dev{t}}^2} +  4 z \gamma L ( 1 + \gamma L) \beta^2   \expect{\norm{\nabla Q(\weight{t})}^2} + z (1 - \beta)^2 \frac{\sigma^2}{n-f} \nonumber \\
    & + 2 z \gamma L ( 1 + \gamma L)\beta^2 \expect{\norm{\drift{t}}^2} - z \expect{\norm{\dev{t}}^2}. \label{eqn:Vt-t}
\end{align}
Upon re-arranging the R.H.S.~in~\eqref{eqn:Vt-t} we obtain that
\begin{align*}
    V_{t+1} - V_t \leq & - \gamma \left(  \left( 1 - 4 \gamma L \right) - 4 z L( 1 + \gamma L) \beta^2   \right) \expect{\norm{\nabla Q(\weight{t})}^2} +  z (1 - \beta)^2 \frac{\sigma^2}{n-f} \nonumber \\
    & + \left( 2 \gamma  \left( 1 + 2 \gamma L  \right) +  z \beta^2 \zeta - z \right)  \expect{\norm{\dev{t}}^2}  + 2 \gamma \left( 1 + \gamma L + zL (1 + \gamma L) \beta^2 \right) \expect{\norm{\drift{t}}^2}. 
\end{align*}
For simplicity, we define
\begin{align}
    A \coloneqq \left( 1 - 4 \gamma L \right) - 4 z L( 1 + \gamma L) \beta^2 , \label{eqn:def_At}
\end{align}
\begin{align}
    B \coloneqq 2 \gamma  \left( 1 + 2 \gamma L  \right) +  z \beta^2 \zeta - z, \label{eqn:def_Bt}
\end{align}
and
\begin{align}
    C \coloneqq 2 \gamma \left( 1 + \gamma L + zL (1 + \gamma L) \beta^2 \right). \label{eqn:def_Ct}
\end{align}
Thus,
\begin{align}
    V_{t+1} - V_t \leq - A \gamma \expect{\norm{\nabla Q(\weight{t})}^2} + B  \expect{\norm{\dev{t}}^2}  + C  \expect{\norm{\drift{t}}^2} +  z (1 - \beta)^2 \frac{\sigma^2}{n-f}. \label{eqn:after_At}
\end{align}
We now analyse below the terms $A$, $B$ and $C$.\\

{\bf Term $A$.} Recall from~\eqref{eqn:gamma_gamma_o} that $ \gamma \leq \gamma_{o} = \frac{1}{18L} $. Upon using this in~\eqref{eqn:def_At}, and the facts that $z = \frac{1}{8L}$ and $\beta^2 < 1$,  we obtain that
\begin{align}
    A \geq 1 - 4 \gamma_o L  - \frac{4L}{8L} ( 1 + \gamma_o L ) \geq \frac{1}{2} - \frac{9\gamma_oL}{2}\geq \frac{1}{4}. \label{eqn:At_3}
\end{align}

{\bf Term $B$.} 
Substituting $\zeta$ from~\eqref{eqn:zeta_expand} in~\eqref{eqn:def_Bt} we obtain that
\begin{align*}
    B &= 2\gamma \left( 1 + 2 \gamma L \right) + z \beta^2 \left( 1 + 5 \gamma L + 4 \gamma^2  L^2 \right) - z \\
    & =  - \left(1 - \beta^2 \right) z +\gamma \left( 2 + 4\gamma L +5 z \beta^2 L + 4 z \beta^2 L \gamma L \right).
\end{align*}
Using the facts that $\beta^2 \leq 1$ and $\gamma \leq \gamma_o \leq \frac{1}{18L}$, and then substituting  $z = \frac{1}{8L}$ we obtain that
\begin{align}\nonumber
    B &\leq \frac{-(1-\beta^2)}{8L} + \gamma \left( 2 + \frac{4}{18} + \frac{5}{8} + \frac{4}{18 \times 8} \right)  \leq \frac{-(1-\beta^2)}{8L} + 3 \gamma\\
    &\leq  \frac{-(1-\beta^2) + 24 \gamma L}{8L} = 0, \label{eqn:Bt_2} 
\end{align}
where in the last equality we used the fact that $1 - \beta^2 = 24 \gamma L$.

{\bf Term $C$.} Substituting $z = \frac{1}{8L}$ in~\eqref{eqn:def_Ct}, and then using the fact that $\beta^2 < 1$, we obtain that
\begin{align*}
    C = 2 \gamma \left( 1 + \gamma L + \frac{1}{8} (1 + \gamma L) \right) \leq \frac{9 \gamma}{4} \left( 1 +\gamma L \right).
\end{align*}
As $\gamma \leq \gamma_o \leq \frac{1}{18 L}$, from above we obtain that
\begin{align}
    C \leq \frac{9 \gamma}{4} \left( 1 +\frac{1}{18} \right) \leq 3 \gamma. \label{eqn:Ct_2}
\end{align}

{\bf Combining terms $A$, $B$ and $C$.} Finally, substituting from~\eqref{eqn:At_3},~\eqref{eqn:Bt_2} and~\eqref{eqn:Ct_2} in~\eqref{eqn:after_At} (and recalling that $z = \frac{1}{8L}$) we obtain that
\begin{align*}
    V_{t+1} - V_t \leq - \frac{ \gamma}{4} \expect{\norm{\nabla Q(\weight{t})}^2} + 3\gamma \expect{\norm{\drift{t}}^2} +   (1 - \beta)^2 \frac{\sigma^2}{8L(n-f)}.
\end{align*}
As the above is true for an arbitrary $t \in [T]$, by taking summation on both sides from $t = 1$ to $t = T$ we obtain that
\begin{align*}
    V_{T+1} - V_1 \leq - \frac{\gamma}{4} \sum_{t = 1}^T \expect{\norm{\nabla Q(\weight{t})}^2}  + 3 \gamma \sum_{t = 1}^T \expect{\norm{\drift{t}}^2} +   (1 - \beta)^2 \frac{\sigma^2}{8L(n-f)} T. 
\end{align*}
Thus,
\begin{align} 
    \frac{\gamma}{4} \sum_{t = 1}^T \expect{\norm{\nabla Q(\weight{t})}^2} \leq V_1 - V_{T+1} + 3 \gamma  \sum_{t = 1}^T \expect{\norm{\drift{t}}^2} +   (1 - \beta)^2 \frac{\sigma^2}{8L(n-f)} T. \label{eqn:lyap_before_rational} 
\end{align}
Note that, as $\beta > 0$, and $1 - \beta^2 = 24 \gamma L$, we have
\begin{align*}
    (1 - \beta)^2 = \frac{\left(1 - \beta^2\right)^2}{\left( 1 + \beta \right)^2} \leq \left( 1 - \beta^2 \right)^2 = 576 \gamma^2 L^2.
\end{align*}
Substituting from above in~\eqref{eqn:lyap_before_rational} we obtain that
\begin{align*}
    \frac{\gamma}{4} \sum_{t = 1}^T \expect{\norm{\nabla Q(\weight{t})}^2} \leq V_1 - V_{T+1} + 3 \gamma  \sum_{t = 1}^T \expect{\norm{\drift{t}}^2} +  \frac{ 576 \gamma^2 L^2  \sigma^2}{8L(n-f)} \, T.
\end{align*}
Multiplying both sides by $4/\gamma$ we obtain that
\begin{align}
    \sum_{t = 1}^T \expect{\norm{\nabla Q(\weight{t})}^2} &\leq \frac{4 \left(V_1 - V_{T+1} \right)}{\gamma} + 12  \sum_{t = 1}^T \expect{\norm{\drift{t}}^2} +  \frac{288  \gamma L \sigma^2}{(n-f)} \, T. \label{eqn:lyap_before_rational_2} 
\end{align}
Next, we use Lemma~\ref{lem:drift} to derive an upper bound on $\sum_{t = 1}^T \expect{\norm{\drift{t}}^2} $.\\

{\bf Invoking Lemma~\ref{lem:drift}.} Recall from Lemma~\ref{lem:drift} that as $F$ is assumed $f$-resilient averaging we have for all $t \in [T]$,
\begin{align*}
    \expect{\norm{\drift{t}}^2} \leq 8 \sigma^2 \lambda^2 (n-f) (1 - \beta)^2 \beta^{2(t - 1)} +  8 \left(\frac{1 - \beta}{1 + \beta}\right) (n - f +1) \lambda^2 \sigma^2.
\end{align*}
By taking summation over $t$ from $1$ to $T$, we obtain that
\begin{align*}
    \sum_{t = 1}^T \expect{\norm{\drift{t}}^2} & \leq 8 \sigma^2 \lambda^2 (n-f) (1 - \beta)^2 \sum_{t = 1}^T \beta^{2(t - 1)} +  8 \left(\frac{1 - \beta}{1 + \beta}\right) (n - f +1) \lambda^2  \, \sigma^2 T \\
    & = 8 \sigma^2 \lambda^2 (n-f) (1 - \beta)^2 \left(\frac{ 1 - \beta^{2T}}{1 - \beta^2} \right) +  8 \left(\frac{1 - \beta}{1 + \beta}\right) (n - f +1) \lambda^2  \, \sigma^2 T \\
    & = 8 \sigma^2 \lambda^2 (n-f) \left(\frac{1 - \beta}{1 + \beta}\right) \left(1 - \beta^{2T} \right) +  8 \left(\frac{1 - \beta}{1 + \beta}\right) (n - f +1) \lambda^2  \, \sigma^2 T 
\end{align*}
As $0 < \beta < 1$, we have $\left(1 - \beta^{2T} \right) \leq 1$. Thus, as $1 \leq n-f \leq (n-f) T$, from above we obtain that
\begin{align}
    \sum_{t = 1}^T \expect{\norm{\drift{t}}^2} \leq 8  \sigma^2 \lambda^2 (n-f) \left(\frac{1 - \beta}{1 + \beta}\right)  +  16 \left(\frac{1 - \beta}{1 + \beta}\right) (n-f) \lambda^2  \sigma^2 T =  24 \sigma^2 \lambda^2 (n - f) T \left(\frac{1 - \beta}{1 + \beta}\right). \label{eqn:sum_drift_t}
\end{align}
As $\beta > 0$, and the fact that $ 1 - \beta^2 = 24 \gamma L$, we have
\begin{align*}
    \frac{1 - \beta}{1 + \beta} = \frac{1 - \beta^2}{( 1 + \beta)^2}  \leq 1 - \beta^2 = 24 \gamma L.
\end{align*}
Substituting the above in~\eqref{eqn:sum_drift_t}, 
we obtain that
\begin{align*}
    \sum_{t = 1}^T \expect{\norm{\drift{t}}^2} \leq (24 \times 24) \sigma^2 \lambda^2 \gamma L  (n - f) T = 576 \sigma^2 \lambda^2 \gamma L  (n - f) T.
\end{align*}
Substituting from above in~\eqref{eqn:lyap_before_rational_2} we obtain that
\begin{align*}
    \sum_{t = 1}^T \expect{\norm{\nabla Q(\weight{t})}^2} \leq & \frac{4 \left(V_1 - V_{T+1} \right)}{\gamma} + (12 \times 576) \sigma^2 \lambda^2 \gamma L (n - f) T  + \frac{288  \gamma L \sigma^2}{(n-f)} \, T 
\end{align*}
Recall that 
\begin{align*}
    a_1 = (12 \times 576) L = 6912 L, ~ \text{ and } a_2 = 288L .
\end{align*}
Thus, from above we obtain that
\begin{align*}
    \sum_{t = 1}^T \expect{\norm{\nabla Q(\weight{t})}^2} \leq \frac{4 \left(V_1 - V_{T+1} \right)}{\gamma} + a_1 \lambda^2 (n - f) \sigma^2 \gamma T +  \frac{ a_2 \sigma^2 }{(n-f)} \, \gamma T. 
\end{align*}
Diving both sides by $T$ we obtain that
\begin{align}
    \frac{1}{T}\sum_{t = 1}^T \expect{\norm{\nabla Q(\weight{t})}^2} \leq \frac{4 \left(V_1 - V_{T+1} \right)}{\gamma T} + a_1 \lambda^2 (n - f) \sigma^2 \, \gamma +  \frac{a_2 \sigma^2}{(n-f)} \, \gamma . \label{eqn:before_Qstar}
\end{align}


{\bf Analysing $V_t$.} Recall that $Q^* = \min_{\weight{} \in \R^d} Q(\weight{})$. Note that for an arbitrary $t$, by definition of $V_t$ in~\eqref{eqn:lyap_func}, 
\[V_t - 2 Q^* = 2 \expect{Q(\weight{t}) - Q^*} + z \expect{\norm{\dev{t}}^2} \geq 0 + z \expect{\norm{\dev{t}}^2} \geq 0.\]
Thus, 
\begin{align}
    V_1 - V_{T+1} = V_1 - 2 Q^* - \left(V_{T+1} -  2 Q^* \right) \leq V_1 - 2 Q^*. \label{eqn:v1-vt}
\end{align}
Moreover,
\begin{align}
    V_1 = 2 Q(\weight{1}) + z \expect{\norm{\dev{1}}^2}. \label{eqn:bnd_v2}
\end{align}
By definition of $\dev{t}$ in~\eqref{eqn:dev}, and the definition of $\AvgMmt{t}$ in~\eqref{eqn:def_avg_mmt},we obtain that
\begin{align*}
    \expect{\norm{\dev{1}}^2} = \expect{\norm{\AvgMmt{1} - \nabla Q(\weight{1})}^2} = \expect{\norm{(1 - \beta) \overline{g}_1 - \nabla Q(\weight{1})}^2} 
\end{align*}
where $\overline{g}_t$, defined in~\eqref{eqn:over_g}, is the average of $n-f$ honest workers' stochastic gradients in step $1$. Expanding the R.H.S.~above we obtain that
\begin{align*}
    \expect{\norm{\dev{1}}^2} = (1 - \beta)^2 \expect{ \norm{\overline{g}_1 - \nabla Q(\weight{1})}^2 } + \beta^2 \norm{\nabla Q(\weight{1})}^2 - 2 \beta (1 - \beta) \iprod{\expect{\overline{g}_1} - \nabla Q(\weight{1})}{ \nabla Q(\weight{1})}.
\end{align*}
Under Assumption~\ref{asp:bnd_var}, we have $\expect{\overline{g}_1} = \nabla Q(\weight{1})$ and $\expect{ \norm{\overline{g}_1 - \nabla Q(\weight{1})}^2} \leq \sigma^2/ (n-f)$. Therefore, 
\begin{align*}
    \expect{\norm{\dev{1}}^2} \leq \frac{(1 - \beta)^2 \sigma^2}{(n-f)} + \beta^2 \norm{\nabla Q(\weight{1})}^2.
\end{align*}
Substituting the above in~\eqref{eqn:bnd_v2} we obtain that
\begin{align*}
    V_1 \leq 2 Q(\weight{1}) + z \left( \frac{(1 - \beta)^2 \sigma^2}{(n-f)} + \beta^2 \norm{\nabla Q(\weight{1})}^2 \right).
\end{align*}
Recall that $(1 - \beta)^2 \leq \left(1 - \beta^2\right)^2 = 576 \gamma^2 L^2$. Using this, and the facts that $\beta^2 < 1$ and $z = \frac{1}{8L}$, we obtain that
\begin{align*}
    V_1 &\leq 2 Q(\weight{1}) +  \frac{1}{8L}\norm{\nabla Q \left( \weight{1} \right)}^2 + \frac{576 \gamma^2 L^2   \sigma^2}{8L(n-f)}\\
    &= 2 Q(\weight{1}) +  \frac{1}{8L}\norm{\nabla Q \left( \weight{1} \right)}^2 + \frac{72 \gamma^2 L   \sigma^2}{(n-f)}.
\end{align*}
Recall that $a_2 = 288L $. Therefore, 
\begin{align*}
    V_1 \leq 2 Q(\weight{1}) + \frac{1}{8L} \norm{\nabla Q \left( \weight{1} \right)}^2 + \frac{a_2 \sigma^2}{4 (n-f)} \, \gamma^2.
\end{align*}
Substituting the above in~\eqref{eqn:v1-vt} we obtain that
\begin{align*}
    V_1 - V_{T+1} \leq 2 Q(\weight{1}) - 2 Q^* +  \frac{1}{8L}\norm{\nabla Q \left( \weight{1} \right)}^2 + \frac{a_2 \sigma^2}{4 (n-f)} \, \gamma^2. 
\end{align*}
Substituting from above in~\eqref{eqn:before_Qstar} we obtain that
\begin{align*}
    \frac{1}{T}\sum_{t = 1}^T \expect{\norm{\nabla Q(\weight{t})}^2} \leq & \frac{4 \left(2 \left(Q(\weight{1}) - Q^* \right) + \frac {\norm{\nabla Q \left( \weight{1} \right)}^2}{8L} \right)}{\gamma T} + \left(\frac{a_2 \sigma^2}{n-f}\right) \frac{\gamma}{T} + a_1 \lambda^2 (n - f) \sigma^2 \, \gamma +  \frac{a_2 \sigma^2}{(n-f)} \, \gamma . 
\end{align*}
Upon re-arranging the terms on R.H.S.~above we obtain that
\begin{align*}
    \frac{1}{T}\sum_{t = 1}^T \expect{\norm{\nabla Q(\weight{t})}^2} \leq & \frac{4 \left(2 \left(Q(\weight{1}) - Q^* \right) + \frac{ \norm{\nabla Q \left( \weight{1} \right)}^2}{8L} \right)}{\gamma T} + \left( a_1 \lambda^2 (n-f)  + \frac{a_2}{n-f} \right) \sigma^2 \gamma + \left( \frac{a_2 \sigma^2}{n-f}\right) \frac{\gamma}{T}. 
\end{align*}
Recall that $a_o = 4 \left(2 \left(Q(\weight{1}) - Q^* \right) + \frac{ \norm{\nabla Q \left( \weight{1} \right)}^2}{8L} \right)$, we obtain that 
\begin{align}
    \frac{1}{T}\sum_{t = 1}^T \expect{\norm{\nabla Q(\weight{t})}^2} \leq & \frac{a_o}{\gamma T} + \left( \frac{a_1 \lambda^2 (n-f)^2 + a_2}{n-f} \right) \sigma^2 \gamma + \left( \frac{a_2 \sigma^2}{n-f}\right) \frac{\gamma}{T}. \label{eqn:after_Qstar}
\end{align}

{\bf Final step.} Recall that
\begin{align*}
    \gamma = \left( \sqrt{\frac{a_o (n-f)}{a_1 \lambda^2 (n-f)^2 + a_2} } \right) \frac{1}{\sigma \sqrt{T}}.
\end{align*}
Substituting this value of $\gamma$ in~\eqref{eqn:after_Qstar} we obtain that
\begin{align*}
    \frac{1}{T}\sum_{t = 1}^T \expect{\norm{\nabla Q(\weight{t})}^2} \leq 2 \sqrt{ \left( a_1 \lambda^2 (n - f) + \frac{ a_2}{n-f}\right)  \frac{a_o \sigma^2}{T} } + \left( \frac{a_2 \sigma}{n-f}\right) \left( \sqrt{\frac{a_o (n-f)}{a_1 \lambda^2 (n-f)^2 + a_2} } \right) \frac{1}{T^{\nicefrac{3}{2}}}.
\end{align*}
Finally, recall from Algorithm~\ref{algo} that $\widehat{\weight{}}$ is chosen randomly from the set of computed parameter vectors $\left\{\weight{1}, \ldots, \, \weight{T} \right\}$. Thus, $\expect{\norm{\nabla Q\left(\widehat{\weight{}} \right)}^2} = \frac{1}{T}\sum_{t = 1}^T \expect{\norm{\nabla Q(\weight{t})}^2}$. Substituting this above proves the theorem.

\end{proof}


\subsection{Proof of Corollary~\ref{cor:resilience}}
\label{app:corollary}


\noindent \fcolorbox{black}{gainsboro!40}{
\parbox{0.97\textwidth}{\centering
\begin{corollary*}
Suppose that assumptions~\ref{asp:lip} and~\ref{asp:bnd_var} hold true. Then, Algorithm~\ref{algo} with an $(f,\lambda)$-resilient averaging rule, and parameters $\gamma_t$, $T$ and $\beta$ as defined in Theorem~\ref{thm:main_conv}, is
is $(f,\epsilon)$-resilient with
\[  \epsilon \in  \mathcal{O}\left( \sqrt{ \frac{\sigma^2}{T} \left( \frac{1}{n-f} + \lambda^2 (n-f) \right)} \right).\]
\end{corollary*}
}}

\begin{proof}
Owing to Theorem~\ref{thm:main_conv}, we have 
\begin{align*}
    \expect{\norm{\nabla Q\left(\widehat{\weight{}}\right)}^2} \leq 2 \sqrt{ \left( a_1 \lambda^2 (n - f) + \frac{ a_2}{n-f}\right)  \frac{a_o \sigma^2}{T} } + \left( \frac{a_2 \sigma}{n-f}\right) \left( \sqrt{\frac{a_o (n-f)}{a_1 \lambda^2 (n-f)^2 + a_2} } \right) \frac{1}{T^{\nicefrac{3}{2}}},
\end{align*} 
where
\begin{align*}
    a_o & = 4 \left(2 \left(Q(\weight{1}) - Q^* \right) + \frac{1}{8L} \norm{\nabla Q \left( \weight{1} \right)}^2 \right), ~ a_1 = 6912L, ~ \text{ and } 
    a_2 =288L.
\end{align*}
Thus, by Definition~\ref{def:resilience}, Algorithm~\ref{algo} is $(f, \, \epsilon)$-resilient where 
\begin{align*}
    \epsilon = \expect{\norm{\nabla Q\left(\widehat{\weight{}}\right)}^2} \leq 2 \sqrt{ \left( a_1 \lambda^2 (n - f) + \frac{ a_2}{n-f}\right)  \frac{a_o \sigma^2}{T} } + \left( \frac{a_2 \sigma}{n-f}\right) \left( \sqrt{\frac{a_o (n-f)}{a_1 \lambda^2 (n-f)^2 + a_2} } \right) \frac{1}{T^{\nicefrac{3}{2}}}.
\end{align*}
Upon ignoring constants, including $a_o$, $a_1$ and $a_2$, and the higher-order term of $T^{\nicefrac{3}{2}}$, we obtain that
\begin{align*}
    \epsilon \in \mathcal{O} \left(\sqrt{ \left( \lambda^2 (n-f) + \frac{1}{n-f}\right)  \frac{\sigma^2}{T} } \right) .
\end{align*} 
Hence, the proof.
\end{proof}

\newpage
\section{Resilience coefficient $\lambda$ for several aggregation rules (Proof of Proposition~\ref{prop:resilience_coeffs})}\label{app:resilience_coefficient}

In this section, we first present a lower bound in Section~\ref{sec:lowebound} on the resilience coefficient for any deterministic $(f, \lambda)$-resilient averaging rule. Then, we present the aggregation rules listed in Table~\ref{tab:resilience_coefficient} and derive their resilience coefficients. More precisely, we compute the resilience coefficients of the following rules. 
\begin{itemize}[leftmargin=*]
    \item \emph{Minimum diameter averaging} (MDA) in Section~\ref{sec:MDA}
    \item \emph{Coordinate-wise trimmed mean} (CWTM) in Section~\ref{sec:CWTM}
    \item \emph{Mean-around-median} (MeaMed) in Section~\ref{sec:Meamed}
    \item \emph{(Multi-)Krum${^*}$} in Section~\ref{sec:Krum}
    \item \emph{Geometric median} (GM) in Section~\ref{sec:GM}
    \item \emph{Coordinate-wise median} (CWMed) in Section~\ref{sec:CWMed}.
\end{itemize}

As an immediate corollary of the result we get for theses aggregation rules, we obtain Proposition~\ref{prop:resilience_coeffs}, that we recall below.

\noindent \fcolorbox{black}{gainsboro!40}{
\parbox{0.97\textwidth}{\centering
\begin{proposition}
Consider an aggregation rule $F \in \{\textnormal{MDA}, \textnormal{CWTM}, \textnormal{MeaMed}, \textnormal{Krum$^*$}, \textnormal{GM}, \textnormal{CWMed}\}$. For any $f < n/2$, there exists a resilience coefficient $\lambda$ for which $F$ is $(f, \, \lambda)$-resilient averaging.
\end{proposition}
}}

Besides computing the aforementioned resilience coefficients, we also discuss the case of \emph{centred clipping} (CC) and \emph{comparative gradient elimination} (CGE) in Section~\ref{sec:CC} and Section~\ref{sec:CGE} respectively.

\subsection{Lower Bound}
\label{sec:lowebound}

\noindent \fcolorbox{black}{gainsboro!40}{
\parbox{0.97\textwidth}{\centering
\begin{proposition}
For $0 \leq f < n$, there cannot exist an $(f, \lambda)$-resilient averaging rule for $\lambda < \frac{f}{n-f}$.
\end{proposition}
}}

\begin{proof}
  Consider an arbitrary value of $f \in \{0, \ldots, \, n - 1\}$. Let $F$ be an $(f, \lambda)$-resilient averaging aggregation rule. Consider a set of $n$ one dimensional vectors $x_1, \ldots, \, x_n$ such that $x_1 = \ldots = x_{n-f} = 0$, and $x_{n-f+1} = \ldots = x_{n} = 1$. Let us first consider a set $S_0 = \left\{1,\ldots, n-f \right\}$. Since $\card{S_0} = n-f$, by Definition \ref{def:rational}, we have
  \begin{equation*}
    \norm{F(x_1, \ldots, \, x_n) - \overline{x}_{S_0}} \leq \lambda \max_{i, j \in {S_0}} \norm{x_i - x_j} = 0.
  \end{equation*}
  Thus, $F(x_1, \ldots, \, x_n) = \overline{x}_{S_0} = 0$. Now, consider another set $S_{1} = \left\{f+1,\ldots, n \right\}$. Note that $\overline{x}_{S_1}=\frac{f}{n-f}$. Thus, 
  \begin{align}
      \norm{F(x_1, \ldots, \, x_n) - \overline{x}_{S_1}} = \frac{f}{n-f}. \label{eqn:low_1}
  \end{align}
  As $F$ is assumed to be an $(f, \lambda)$-resilient averaging rule, by Definition \ref{def:rational} we have 
  \begin{equation*}
      \norm{F(x_1, \ldots, \, x_n) - \overline{x}_{S_1}} \leq \lambda \max_{i, j \in {S_1}} \norm{x_i - x_j} = \lambda.
  \end{equation*}
  If $\lambda < \frac{f}{n-f}$ then the above contradicts~\eqref{eqn:low_1}. This concludes the proof.
\end{proof}


\subsection{Minimum Diameter Averaging (MDA)}
\label{sec:MDA}
Given a set of $n$ vectors $x_1, \ldots, \, x_n$, the MDA algorithm, originally proposed in~\cite{rousseeuw1985multivariate} and reused in~\cite{brute_bulyan}, first chooses a set $S^{*}$ of cardinality $n-f$ with the smallest {\em diameter}, i.e., 
\begin{equation}
    S^{*} \in \argmin_{\underset{\card{S} = n-f}{S \subset \{ 1, \ldots, \, n} \} } \left\{\max_{i,j\in S} \norm{x_i - x_j} \right\}. \label{eqn:def_S*}
\end{equation}
Then the algorithm outputs, the average of the inputs in set $S^{*}$. Specifically it outputs  
\begin{equation}
    \text{MDA}(x_1, \ldots, \, x_n) \coloneqq \frac{1}{n-f} \sum_{i \in S^{*}} x_i.
\end{equation}

\noindent \fcolorbox{black}{gainsboro!40}{
\parbox{0.97\textwidth}{\centering
\begin{proposition}
   If $f < n/2$ then \textnormal{MDA} is an $(f, \, \lambda)$-resilient averaging rule for $\lambda = \frac{2f}{n-f}$.
\end{proposition}
}}

\begin{proof}
    Let $S$ be an arbitrary subset of $\{1, \ldots, \, n\}$ such that $\card{S} = n - f$. To prove the proposition we first show that
    \begin{align*}
        \norm{\text{MDA}(x_1, \ldots, \, x_n) - \overline{x}_{S}} \leq \frac{f}{n-f} \max_{i \in S, j \in S^{*}} \norm{x_i-x_j}
    \end{align*}
    where $\overline{x}_S \coloneqq \frac{1}{\mnorm{S}} \sum_{i \in S} x_i$.\\
    
    \noindent In doing so , we note that $\card{S^{*} \setminus S} = \card{S^{*}\cup S} - \card{S} \leq n - (n-f) = f$. The same observation holds for $\card{S \setminus S^*}$. Hence we obtain that
    \begin{align}
      \norm{\text{MDA}(x_1, \ldots, \, x_n) - \overline{x}_{S}} &= \norm{\frac{1}{n-f} \sum_{i \in S^{*}} x_i - \frac{1}{n-f} \sum_{i \in S} x_i} = \frac{1}{n-f} \norm{\sum_{i \in S^{*} \setminus S} x_i-\sum_{i \in S \setminus S^{*}} x_i} \nonumber \\
      & \leq \frac{\max(\card{S^{*} \setminus S};\card{S \setminus S^{*}})}{n-f} \max_{i \in S^*, j \in S} \norm{x_i-x_j } \leq \frac{f}{n-f} \max_{i \in S, j \in S^{*}} \norm{x_i-x_j}. \label{eq:MDA1}
    \end{align}
    As we assume that $f < n/2$, we also have
    \begin{equation*}
    \card{S^{*} \cap S} = \card{S} + \card{S^{*}}- \card{S^{*} \cup S} \geq (n-f) + (n-f) - n \geq n - 2f > 0.
    \end{equation*}
    Therefore, $S^{*} \cap S \neq \emptyset$. Let $i^*$ be an arbitrary index that belongs to both $S$ and $S^{*}$. From triangle inequality, we obtain that, for any $i^{\dag} \in S^{*}$ and $j^{\dag} \in S$, 
    \begin{align*}
        \norm{x_{i^{\dag}}- x_{j^{\dag}}} \leq \norm{x_{i^{\dag}} - x_{i^*}} + \norm{x_{i^*} - x_{j^{\dag}}} \leq \max_{i,j\in S^{*}} \norm{x_i - x_j} + \max_{i,j\in S} \norm{x_i - x_j}.
    \end{align*}
    By definition of $S^{*}$ in~\eqref{eqn:def_S*}, $\max_{i,j\in S^{*}} \norm{x_i - x_j} \leq \max_{i,j\in S} \norm{x_i - x_j}$. Thus, from above we obtain that
    \begin{align}
        \norm{x_{i^{\dag}}- x_{j^{\dag}}} \leq 2\max_{i,j\in S_{}} \norm{x_i - x_j}. \label{eq:MDA2}
    \end{align}
    As $i^{\dag}$ and $j^{\dag}$ above are arbitrary elements in $S^{*}$ and $S$, respectively, from above we obtain that
    \begin{equation*}
        \max_{i \in S, j \in S^{*}} \norm{x_i-x_j} \leq 2\max_{i,j\in S_{}} \norm{x_i - x_j}.
    \end{equation*}
    Combining the above with (\ref{eq:MDA2}) we obtain that 
    \begin{equation*}
       \norm{\text{MDA}(x_1, \ldots, \, x_n) - \overline{x}_{S}} \leq \frac{2 f}{n-f} \max_{i,j\in S_{}} \norm{x_i - x_j}.
    \end{equation*}
    As $S$ is an arbitrary subset of $[n]$ of size $n - f$, the above proves the proposition.
\end{proof}


\subsection{Coordinate-Wise Trimmed Mean (CWTM)}
\label{sec:CWTM}
Let $x \in \mathbb{R}^d$, we denote by $[x]_k$, the $k$-th coordinate of $x$. Given the input vectors $x_1, \ldots, \, x_n$ (in $\R^d$), we let $\tau_k$ denote a permutation on $[n]$ that sorts the $k$-coordinate of the input vectors in non-decreasing order, i.e., $[x_{\tau_k(1)}]_k\leq [x_{\tau_k(2)}]_k \leq\ldots \leq [x_{\tau_k(n)}]_k$. Then, the CWTM of $x_1, \ldots, \, x_n$, denoted by $\text{CWTM}(x_1, \ldots, x_n)$, is a vector in $\R^d$ whose $k$-th coordinate is defined as follows,
\begin{equation*}
    \left[\text{CWTM}(x_1, \ldots, x_n)\right]_k \coloneqq \frac{1}{n-2f} \sum_{j \in [f+1,n-f]} [x_{\tau_k(j)}]_k. \label{eqn:def_cwtm}
\end{equation*}

To obtain the resilience coefficient of CWTM, we recall show in Lemma~\ref{lemma:diameter} below how the {\em diameter} of a set of vectors is related to their {\em coordinate-wise diameter}. This lemma also proves useful to other coordinate-wise aggregation rules, e.g., CWMed.

\begin{lemma}
\label{lemma:diameter}
   For a non-empty set of $d$-dimensional vectors $S$, we have
   \begin{equation*}
        \sqrt{ \sum_{k = 1}^{d} \left(\max_{i, j \in S} \card{[x_i]_k - [x_j]_k} \right)^2} \leq \min \left\{2\sqrt{\card{S}}, \sqrt{d}\right\}\max_{i, j \in S}  \norm{x_i - x_j}.
   \end{equation*}
\end{lemma}
\begin{proof}
   Special case of Lemma 18 in~\cite{collaborativeElMhamdi21} for $r = 2$.
\end{proof}

We can now formally state the proposition proving that CWTM is an $(f,\lambda)$-resilient averaging rule.

\noindent \fcolorbox{black}{gainsboro!40}{
\parbox{0.97\textwidth}{\centering
\begin{proposition}
   If $f < n/2$ then \textnormal{CWTM} is an $(f, \, \lambda)$-resilient averaging rule for $\lambda =\frac{f}{n-f} \min \left\{2\sqrt{n-f}, \sqrt{d}\right\} $.
\end{proposition}
}}

\begin{proof}
   The idea of the proof is similar to that of Theorem 5 in~\cite{collaborativeElMhamdi21}. Consider an arbitrary set $S \subseteq [n]$ such that $\card{S} = n - f > f$. For each coordinate $k \in [d]$, let $\pi^S_k$ denote a permutation on $S$ such that $[x_{\pi^S_k(1)}]_k \leq [x_{\pi^S_k(2)}]_k\leq \ldots \leq [x_{\pi^S_k(\card{S})}]_k$.  Let $c =  \text{CWTM}(x_1, \ldots, x_n)$. Then, by Definition of the permuations $\tau_k$, for each $k$ we have that
   \begin{equation}
       \frac{1}{n-2f}  \sum_{i = 1}^{n-2f} [x_{\pi^S_k(i)}]_k  \leq \frac{1}{n-2f} \sum_{j = f+1}^{n-f} [x_{\tau_k(j)}]_k = [c]_k. \label{eqn:S_c-k}
   \end{equation}
  Note that for all $j \in S$ and $k \in [d]$, we have
  \begin{align*}
      [x_j]_k &= [x_j]_k + \frac{1}{n-2f}  \sum_{i = 1}^{n-2f} [x_{\pi^S_k(i)}]_k - \frac{1}{n-2f}  \sum_{i = 1}^{n-2f} [x_{\pi^S_k(i)}]_k \\
      & = \frac{1}{n-2f}  \sum_{i = 1}^{n-2f} [x_{\pi^S_k(i)}]_k + \frac{1}{n-2f} \sum_{i = 1}^{n-2f} \left([x_j]_k - [x_{\pi^S_k(i)}]_k \right).
  \end{align*}
  Substituting from~\eqref{eqn:S_c-k} above we obtain for all $j \in S$ and $k \in [d]$ that
   \begin{equation}
       [x_j]_k \leq [c]_k + \frac{1}{n-2f} \sum_{i = 1}^{n-2f} \left([x_j]_k - [x_{\pi^S_k(i)}]_k \right) \leq [c]_k + \max_{l, m \in S} \card{[x_l]_k - [x_m]_k}. \label{eqn:S_c-k_2}     
   \end{equation}
   Recall that $\overline{x}_{S} := \nicefrac{1}{\mnorm{S}} \sum_{i \in S} x_i$. From~\eqref{eqn:S_c-k} and~\eqref{eqn:S_c-k_2} we obtain that
   \begin{align}
       [\overline{x}_{S}]_k &= \frac{1}{n-f} \sum_{i \in [n-f]} [x_{\pi^S_k(i)}]_k = \frac{1}{n-f}  \sum_{i = 1}^{n-2f} [x_{\pi^S_k(i)}]_k + \frac{1}{n-f}  \sum_{i = n - 2f + 1}^{n-f} [x_{\pi^S_k(i)}]_k \nonumber \\
       &\leq  \frac{n-2f}{n-f}[c]_k + \frac{f}{n-f} ([c]_k + \max_{i, j \in S} \card{[x_i]_k - [x_j]_k}) = [c]_k + \frac{f}{n-f} \max_{i, j \in S} \card{[x_i]_k - [x_j]_k}. \label{eqn:cwtm_overS_1}
   \end{align}
   ~
   
   Now, similar to~\eqref{eqn:S_c-k}, we obtain for all $k \in [d]$ that
   \begin{align}
       \frac{1}{n-2f}  \sum_{i = f+1}^{n-f} [x_{\pi^S_k(i)}]_k  \geq \frac{1}{n-2f} \sum_{j = f+1}^{n-f} [x_{\tau_k(j)}]_k = [c]_k. \label{eqn:S_c-k_geq}
   \end{align}
   In a similar manner to~\eqref{eqn:S_c-k_2}, we obtain for all $j \in S$ and $k \in [d]$ that
   \begin{align}
       [x_j]_k \geq [c]_k - \max_{l, m \in S} \card{[x_l]_k - [x_m]_k}. \label{eqn:S_c-k_geq_2}
   \end{align}
   From~\eqref{eqn:S_c-k_geq} and~\eqref{eqn:S_c-k_geq_2}, in a similar manner to~\eqref{eqn:cwtm_overS_1}, we obtain that 
   \begin{align}
       [\overline{x}_{S}]_k \geq [c]_k - \frac{f}{n-f} \max_{i, j \in S} \card{[x_i]_k - [x_j]_k}. \label{eqn:cwtm_overS_2}
   \end{align}
~

Owing to~\eqref{eqn:cwtm_overS_1} and~\eqref{eqn:cwtm_overS_2} we obtain that, for all $k \in [d]$,
  \begin{equation*}
     \mnorm{[\overline{x}_{S}]_k - [c]_k} \leq \frac{f}{n-f} \max_{i, j \in S} \mnorm{[x_i]_k - [x_j]_k}.
  \end{equation*}
Thus, 
\begin{align*}
    \norm{\overline{x}_{S} - c} = \sqrt{ \sum_{k\in [d]} \mnorm{[\overline{x}_{S}]_k - [c]_k}^2} \leq \frac{f}{n-f} \sqrt{ \sum_{k\in [d]} \left(\max_{i, j \in S} \mnorm{[x_i]_k - [x_j]_k} \right)^2}.
\end{align*}
Recall that $c = \text{CWTM}(x_1, \ldots, x_n)$. Thus, using Lemma~\ref{lemma:diameter} we obtain that
  \begin{equation*}
      \norm{\overline{x}_{S} - \text{CWTM}(x_1, \ldots, x_n)} \leq \frac{f}{n-f} \min \left\{2\sqrt{n-f}, \sqrt{d}\right\}  \max_{i, j \in S}  \norm{x_i - x_j}.
  \end{equation*}
  As $S$ is an arbitrary subset of $[n]$ of size $n - f$, concludes the proof.
\end{proof}


\subsection{Mean around Median (MeaMed)}
\label{sec:Meamed}
Let $x \in \mathbb{R}^d$, we denote by $[x]_k$, the $k$-th coordinate of $x$. Given the input vectors $x_1, \ldots, \, x_n$ (in $\R^d$), MeaMed computes the average of  the $n-f$ closest elements to the median in each dimension. Specifically, for each $k \in [d]$, $m \in [n]$, let $i_{m;k}$ be the index of the input vector with $k$-th coordinate that is $m$-th closest to $\textnormal{Median}([x_1]_k,\dots,[x_n]_k)$. Let $C_k$ be the set of $n-f$ indices defined as
\[C_k = \{i_{1;k}, \ldots, \, i_{n-f;k}\}.\] Then we have
\begin{equation*}
    [\text{MeaMed}(x_1, \ldots, x_n)]_k = \frac{1}{n-f} \sum_{i \in C_{k}} [x_i]_k,
\end{equation*}
where $\text{MeaMed}(x_1, \ldots, x_n)$ denotes the output of the aggregation rule.

\noindent \fcolorbox{black}{gainsboro!40}{
\parbox{0.97\textwidth}{\centering
\begin{proposition}
    If $f < n/2$, then \textnormal{MeaMed} is an $(f, \, \lambda)$-resilient averaging for $\lambda =\frac{2f}{n-f} \min \left\{2\sqrt{n-f}, \sqrt{d}\right\} $.
\end{proposition}
}}

\begin{proof}
   Consider an arbitrary set $S$ such that $\card{S} = n -f$. Since $\card{S} >n/2$, by the definition of the median, for each $k \in [d]$, we have
   \begin{equation*}
       \min_{i \in S} [x_i]_k \leq \text{Median}([x_1]_k, \ldots, [x_n]_k) \leq \max_{i \in S} [x_i]_k.
   \end{equation*}
   Accordingly, for any $j \in S$ and $k \in [d]$, we have
   \begin{equation}
   \label{eq:meamed1}
       \absv{\text{Median}([x_1]_k, \ldots, [x_n]_k) - [x_j]_k} \leq  \max_{i \in S} [x_i]_k -  \min_{i \in S} [x_i]_k.
   \end{equation}
   In particular, this means that there exist at least $n - f$ vectors within $[x_1]_k, \ldots, [x_n]_k$, whose absolute deviation from $\text{Median}([x_1]_k, \ldots, [x_n]_k)$ is upper-bound by $\max_{i \in S} [x_i]_k - \min_{i \in S} [x_i]_k$. Therefore, by the definition of $C_k$, for any $j \in C_k$, we have
   \begin{equation}
   \label{eq:meamed2}
       \absv{Median([x_1]_k, \ldots, [x_n]_k) - [x_j]_k} \leq  \max_{i \in S} [x_i]_k -  \min_{i \in S} [x_i]_k.
   \end{equation}
   Combining (\ref{eq:meamed1}) and (\ref{eq:meamed2}), then implies that for any $l \in S$ and $m \in C_k$, we have
   \begin{equation*}
       \card{[x_l]_k - [x_m]_k} \leq 2 \max_{i,j \in S} ([x_i]_k- [x_j]_k).
   \end{equation*}
   Also note that $\card{S \setminus C_k} = \card{C_k \setminus S} = \card{C_k \cup S} - \card{S} \leq n - (n -f) = f$. Hence we get 
   \begin{align*}
       \card{[\text{MeaMed}(x_1, \ldots, x_n)]_k - [\overline{x}_{S}]_k} &= \frac{1}{n-f} \card{\sum_{i \in C_k} [x_i]_k - \sum_{i \in S} [x_i]_k} \\
       &= \frac{1}{n-f} \card{\sum_{i \in C_k \setminus S} [x_i]_k - \sum_{i \in S \setminus C_k} [x_i]_k}\\
       &\leq \frac{2f}{n-f} \max_{i,j \in S} ([x_i]_k- [x_j]_k).
   \end{align*}
   Finally, by using Lemma \ref{lemma:diameter}, we get
   \begin{align}
    \norm{\overline{x}_S - \text{MeaMed} (x_1, \ldots, x_n)} 
    &\leq \frac{2f}{(n-f)} \sqrt{ \sum_{k\in [d]} \left(\max_{i, j \in S} \card{[x_i]_k - [x_j]_k} \right)^2}\\ &\leq \frac{2f}{(n-f)} \min \{2\sqrt{n-f}, \sqrt{d}\}\max_{i, j \in S}  \norm{x_i - x_j}.
\end{align}
The above concludes the proof.
\end{proof}


\newcommand{\q}{q}
\newcommand{\MkrumSet}{M}
\subsection{(Multi-)Krum$^*$}
\label{sec:Krum}
In this section, we study a slight adaptation of the Multi-Krum algorithm first introduced in~\cite{krum}. This adaptation, called Multi-Krum$^*$, is mainly changing one step of the procedure to enhance the tolerance of the method from $f < (n - 2)/2 $ (needed for the original method) to $f < n/2$ (i.e., the optimal tolerance threshold).

Essentially, given the input vectors $x_1, \ldots, \, x_n$, Multi-Krum$^*$ outputs an average of the vectors that are the closest to their neighbors upon discarding $f$\footnote{As opposed to $f+1$ in the original version.} farthest vectors. Specifically, for each $i \in [n]$ and $k \in [n-1]$, let $i_k \in [n] \setminus \{i\}$ be the index of the $k$-th closest input vector from $x_i$, i.e., we have $\norm{x_i - x_{i_1}} \leq \ldots \leq \norm{x_i - x_{i_{n-1}}}$ with ties broken arbitrarily. Let $C_i$ be the set of $n-f-1$ closest vectors to $x_i$, i.e., 
\[C_i = \{i_{1}, \ldots, \, i_{n-f-1}\}.\] Then, for each $i \in [n]$, we define $score(i) \coloneqq \sum_{j\in C_i} \norm{x_i - x_j}^2 $. 
Finally, Multi-Krum$^*_\q$ outputs the average of $\q$ input vectors with the smallest scores, i.e., 
\begin{align*}
    \text{Multi-Krum}^{*}_\q \left(x_1, \ldots, \, x_n \right) = \frac{1}{\q} \sum_{i \in \MkrumSet(\q)} x_i,
    \label{eqn:def_krum}
\end{align*}
where $ \MkrumSet(\q)$ is the set of $\q$ vectors with the smallest scores. \underline{We call by Krum$^*$} the special case of Multi-Krum$^*_\q$ for $\q = 1$.

Before analyzing Multi-Krum$^*_\q$, we prove the following lemma.

\begin{lemma}
\label{lemma:mkrum}
  Consider a set $S \subset [n]$ such that $\card{S}= n-f$. Suppose $\q \leq n -f$. For any $k \in \MkrumSet(\q)$ and $l \in S$, we have
  \begin{equation*}
      \norm{x_k-{x}_{l}} \leq \left(1+\sqrt{\frac{n-f}{n-2f}}\right) \max_{i, j \in S} \norm{x_i - x_j}.
      \label{eqn:krum_proof}
  \end{equation*}
  \begin{proof}
   To demonstrate this result, we study two cases separately; {\bf case i)} $k \in S$, and {\bf case ii)} $k \not \in S$.
  
  \noindent {\bf Case i)} Let $l \in S$, if $k \in S$, by definition we have
  \begin{align}
    \norm{x_{k} - {x}_l}  \leq  \max_{i, j \in S} \norm{x_i - x_j}. \label{eqn:case_1}
  \end{align}
  Thus,~\eqref{eqn:krum_proof} trivially holds in case i).
  
  \noindent {\bf Case ii)} Let us now consider that $k \not \in S$. Since $\card{\MkrumSet(\q)} = \q \leq n-f$, there exists at least an index $m \in S$ such that $m \notin \MkrumSet(\q)$. Then by the definitions of the score function $score(\cdot)$ and of the set $C_m$, we get that
  \begin{equation}
      score(m) = \sum_{j \in C_{m}} \norm{x_m - x_j}^2 \leq \sum_{j\in S} \norm{x_{m} - x_j}^2 \leq  (n-f) \max_{i, j \in S} \norm{x_i - x_j}^2.  \label{eqn:score_k}
  \end{equation}
  Since $m \notin \MkrumSet(\q)$, we have $score(k) \leq score(m)$. Accordingly, we have that
  \begin{align}
    score(k) = \sum_{j\in C_{m}} \norm{x_{k} - x_j}^2 \leq score(m). \label{eqn:score_istar}
  \end{align}
  Note that $\card{C_{k} \cap S} = \card{C_{k}} + \card{S} -  \card{C_{k} \cup S} \geq (n-f) + (n-f) - n = n - 2f$. As $f < n/2$, we get $C_{k} \cap S \neq \emptyset$. Now, as $C_{k} \cap S \subseteq C_k$, we have $\sum_{j\in C_{k} \cap S} \norm{x_{k} - x_j}^2 \leq \sum_{j\in C_{k}} \norm{x_{k} - x_j}^2$. Thus, from~\eqref{eqn:score_istar} we obtain that 
  \begin{align*}
      \sum_{j\in C_{k} \cap S} \norm{x_{k} - x_j}^2 \leq score(m).
  \end{align*}
  Substituting from~\eqref{eqn:score_k} above, we obtain that
  \begin{align*}
      \sum_{j\in C_{k} \cap S} \norm{x_{k} - x_j}^2 \leq (n-f) \max_{i, j \in S} \norm{x_i - x_j}^2.
  \end{align*}
  As $\card{C_{k} \cap S} \geq n - 2f$, $\sum_{j\in C_{k} \cap S} \norm{x_{k} - x_j}^2 \geq (n-2f) \min_{j\in S} \norm{x_{k} - x_j}^2$.  Thus, from above we obtain that
  \begin{align*}
    (n-2f) \min_{j\in S} \norm{x_{k} - x_j}^2 \leq (n-f) \max_{i, j \in S} \norm{x_i - x_j}^2.
  \end{align*}
  This implies that
  \begin{align}
     \min_{j\in S} \norm{x_{k} - x_j} \leq \sqrt{\frac{n-f}{n-2f}} \, \max_{i, j \in S} \norm{x_i - x_j}. \label{eqn:min_j_S}
  \end{align}
  Let $l \in S$ and $j^* \in \arg \min_{j \in S}  \norm{x_{k} - x_j}$. By the triangle inequality, we then obtain that
  \begin{align}
      \norm{x_{k} - {x}_l}  = \norm{x_k - x_{j^*} + x_{j^*} - x_l} \leq \norm{x_k - x_{j^*}} + \norm{x_{j^*} - x_l} \label{eqn:jstar_S}
  \end{align}
  Substituting from above in~\eqref{eqn:min_j_S} and using the fact that $\norm{x_{j^*} - x_l} \leq \max_{i, j \in S} \norm{x_i - x_j}$, we then obtain that
  \begin{align}
      \norm{x_{k} - {x}_l} \leq \left(1+\sqrt{\frac{n-f}{n-2f}}\right) \max_{i, j \in S} \norm{x_i - x_j}. \label{eqn:case_2}
  \end{align}
  The above proves~\eqref{eqn:krum_proof} in case ii).
  
  \noindent As~\eqref{eqn:krum_proof} holds true in either case (see~\eqref{eqn:case_1} and~\eqref{eqn:case_2}), the lemma holds true.
  \end{proof}
\end{lemma}

We present below a proposition characterizing the resilient averaging property of Multi-Krum$^*_q$. Note that the resilience coefficient of Krum$^*$ can be immediately derived from this proposition by substituting $q = 1$.

\noindent \fcolorbox{black}{gainsboro!40}{
\parbox{0.97\textwidth}{\centering
\begin{proposition}
   If $f < n/2$, and $\q \leq n-f$ then Multi-Krum$^*_q$ is an $(f, \, \lambda)$-resilient averaging rule for $$\lambda = \left(1 + \sqrt{\frac{n-f}{n-2f}}\right) \cdot \min \left\{1,\frac{n-q}{n-f}\right\}.$$
\end{proposition}
}}

\begin{proof}
   Consider a set of vectors $S$ such that $\card{S} = n-f$. As in the proof of Lemma~\ref{lemma:mkrum}, we consider two different cases separately; {\bf case i)} $\q \leq f$, and {\bf case ii)} $\q > f$.
   
   \noindent {\bf Case i)} Let $\q \leq f$. By triangle inequality and Lemma \ref{lemma:mkrum}, we obtain that 
   \begin{align*}
       \norm{\text{Multi-Krum}^{*}_\q \left(x_1, \ldots, \, x_n \right) - \overline{x}_S} &= \norm{\frac{1}{\q} \sum_{i \in \MkrumSet(\q)} x_i - \overline{x}_S} \leq \frac{1}{\q} \sum_{i \in \MkrumSet(\q)} \norm{x_i - \overline{x}_S} \leq \frac{1}{\q} \sum_{i \in \MkrumSet(\q)} \norm{x_i - \frac{1}{n-f}\sum_{j\in S}x_j} \\
       &\leq \frac{1}{\q (n-f)} \sum_{i \in \MkrumSet(\q)} \sum_{j \in S}\norm{x_i - x_j} \\
       &\leq \frac{1}{\q (n-f)} \sum_{i \in \MkrumSet(\q)} \sum_{j \in S} \left(1+\sqrt{\frac{n-f}{n-2f}}\right) \max_{i, j \in S} \norm{x_i - x_j}\\
       &= \left(1+\sqrt{\frac{n-f}{n-2f}}\right) \max_{i, j \in S} \norm{x_i - x_j}. 
   \end{align*}
   Thus, the proposition holds true in case i).
   
   \noindent {\bf Case ii)} Let us now consider $\q > f$. We have $\card{S \cap \MkrumSet(\q)} = \card{S} + \card{\MkrumSet(\q)} - \card{S \cup \MkrumSet(\q)} \geq (n - f) + q - n = q -f > 0$. Therefore, there exists a set $P$ with cardinality $q-f$ such that $P \subset S \cap \MkrumSet(\q)$. Hence we get
   \begin{align}
       &\norm{\text{Multi-Krum}^{*}_\q \left(x_1, \ldots, \, x_n \right) - \overline{x}_S} = \norm{\frac{1}{n-f}\sum_{i \in S}x_i - \frac{1}{\q} \sum_{i \in \MkrumSet(\q)} x_i} \\
       &=
       \norm{\frac{1}{n-f}\sum_{i \in S\setminus P}x_i - \frac{1}{\q} \sum_{i \in \MkrumSet(\q)\setminus P} x_i - \left(\frac{1}{\q}-\frac{1}{n-f}\right)\sum_{i \in P}x_i}\\
       &= \frac{1}{q(n-f)}\norm{q\sum_{i \in S\setminus P}x_i - \left((n-f) \sum_{i \in \MkrumSet(\q)\setminus P} x_i + \left(n-f-\q\right)\sum_{i \in P}x_i\right)} = \frac{1}{q(n-f)} \norm{A-B}, \label{eq:krum_proof}
   \end{align}
   where 
   \begin{equation}
       A \coloneqq q\sum_{i \in S\setminus P}x_i \quad \text{and} \quad B \coloneqq \left((n-f) \sum_{i \in \MkrumSet(\q)\setminus P} x_i + \left(n-f-\q\right)\sum_{i \in P}x_i\right).
   \end{equation}
 Since $\card{S\setminus P} = (n-f) - (q-f) = n - q$, $A$ is a sum of $q(n-q)$ (potentially repetitive) vectors all of which belong to $S$. Also, $f(n-f)+(n-f-q)(q-f) = q(n-q)$. Thus, $B$ is also a sum of $ q(n-q)$ (potentially repetitive) vectors all of which belong to $\MkrumSet(\q)$. We now match each vector in $A$ to a vector in $B$. Using the triangle inequality and Lemma \ref{lemma:mkrum}, we the obtain
   \begin{equation}
       \norm{A-B} \leq q (n-q) \left(1+\sqrt{\frac{n-f}{n-2f}}\right) \max_{i, j \in S} \norm{x_i - x_j}. 
   \end{equation}
   Combining above with \eqref{eq:krum_proof}, we then obtain 
   \begin{equation}
       \norm{\text{Multi-Krum}^{*}_\q \left(x_1, \ldots, \, x_n \right) - \overline{x}_S} \leq \frac{n-q}{n-f} \left(1+\sqrt{\frac{n-f}{n-2f}}\right) \max_{i, j \in S} \norm{x_i - x_j}.
   \end{equation}
   This shows that the proposition holds true in case ii). 
   
   Combing the conclusions for cases i) and ii) concludes the proof
\end{proof}

\subsection{Geometric Median (GM)}
\label{sec:GM}
 For input vectors $x_1, \ldots, \, x_n$, their geometric median, denoted by $\text{GM}(x_1, \ldots, x_n)$, is defined to be a vector that minimizes the sum of the distances to these vectors. Specifically, we have
 \begin{equation*}
     \textnormal{GM}(x_1, \ldots, x_n) \in \argmin_{z \in \R^d} \sum_{i = 1}^n \norm{z-x_i}.
 \end{equation*}
 For obtaining the resilience coefficient of GM, we make use of the following three lemmas. Below, we denote by $\text{Conv}(x_1, \ldots, \, x_n)$ the {\em convex hull} of $x_1, \ldots, \, x_n$, i.e., 
 \begin{align*}
     \textnormal{Conv}(x_1, \ldots, \, x_n) = \left\{ \sum_{i = 1}^n a_i x_i ~ \vline ~ \sum_{i = 1}^n a_i = 1, ~ a_i \geq 0, \, \forall i \in [n] \right\}
 \end{align*}

\begin{lemma}
\label{lem:dia_cov}
Let $y$ and $z$ be any two points in $\textnormal{Conv}(x_1, \ldots, \, x_n)$. Then, $\norm{y - z} \leq \max_{i, \, j \in [n]} \norm{x_i - x_j}$.
\end{lemma}
\begin{proof}
   By definition, suppose that $y = \sum_{i = 1}^n a_i x_i$ and $z = \sum_{i = 1}^n b_i x_i$ such that $\sum_{i = 1}^n a_i = 1$, $\sum_{i = 1}^n b_i = 1$, and $a_i\geq0$, $b_i\geq0$ for $i\in[n]$. We then obtain
   \begin{align*}
       \norm{y-z} &= \norm{\sum_{i = 1}^n a_i x_i - z} = \norm{\sum_{i = 1}^n a_i (x_i - z)} = \norm{\sum_{i = 1}^n a_i \left(x_i - \sum_{j = 1}^n b_j x_j\right)} = \norm{\sum_{i = 1}^n a_i \left(\sum_{j = 1}^n b_j (x_i-x_j)\right)}.
   \end{align*}
   Using triangle inequality we obtain that
   \begin{align*}
       \norm{y-z} &\leq \sum_{i = 1}^n a_i \left(\sum_{j = 1}^n b_j \norm{x_i-x_j}\right)
       \leq  \sum_{i = 1}^n a_i \left(\sum_{j = 1}^n b_j \max_{k, \, l \in [n]} \norm{x_k - x_l}\right) = \max_{k, \, l \in [n]} \norm{x_k - x_l} \sum_{i = 1}^n a_i \left(\sum_{j = 1}^n b_j \right) \\
       & = \max_{k, \, l \in [n]} \norm{x_k - x_l}.
   \end{align*}
   Hence, the proof.
\end{proof}

\begin{lemma} [Proposition 6 in~\citep{Geomed21}]
\label{lemma:GM_convex}
For any input vectors $x_1, \ldots, x_n \in \mathbb{R}^d$, the following holds true: $$\textnormal{GM}(x_1, \ldots, x_n) \in \textnormal{Conv}(x_1, \ldots, \, x_n).$$
\end{lemma}

For a non-empty set $S \subseteq [n]$. In the remaining, we denote by $\{x_i\}_{i \in S}$ the set of vectors which index is in $S$, i.e., $\{x_i, \, i \in S\}$.

\begin{lemma}[Theorem 1 (Part 1) in~\citep{Geomed21}]
\label{lemma:GM_resilient}
    For any set $S \subseteq [n]$ such that $\card{S} > n/2$,
    \begin{equation*}
      \norm{\textnormal{GM}(x_1, \ldots, x_n) - \textnormal{GM}\left(\{x_i\}_{i\in S} \right)}\leq \frac{1}{\sqrt{1-\frac{(n-\card{S})^2}{\card{S}^2}}} \max_{j \in S} \norm{x_j - \textnormal{GM}\left(\{x_i\}_{i\in S} \right)}.
    \end{equation*}
\end{lemma}

By combing the above lemmas, we can devise the following result.

\noindent \fcolorbox{black}{gainsboro!40}{
\parbox{0.97\textwidth}{\centering
\begin{proposition}
   If $f < n/2$ then the \textnormal{GM} is an $(f, \, \lambda)$-resilient averaging rule for $\lambda =1+\frac{n-f}{\sqrt{(n-2f)n}}.$
\end{proposition}
}}

\begin{proof}
   Consider any set $S \subseteq [n]$ such that $\card{S} = n - f > n/2$. By triangle inequality we obtain that
   \begin{align*}
       \norm{\textnormal{GM}(x_1, \ldots, x_n) - \overline{x}_{S}} \leq \norm{\textnormal{GM}(x_1, \ldots, x_n)-\textnormal{GM}\left(\{x_i\}_{i\in S} \right)} + \norm{\textnormal{GM}\left(\{x_i\}_{i\in S} \right) - \overline{x}_{S}}.
   \end{align*}
   Substituting from Lemma~\ref{lemma:GM_resilient} above we obtain that 
   \begin{align}
      \norm{\textnormal{GM}(x_1, \ldots, x_n) - \overline{x}_{S}} \leq \frac{n-f}{\sqrt{(n-2f)n}} \max_{j \in S} \norm{x_j - \textnormal{GM}\left(\{x_i\}_{i\in S} \right)} + \norm{\textnormal{GM}\left(\{x_i\}_{i\in S} \right) - \overline{x}_{S}}. \label{eqn:before_lemma_cov}
   \end{align}
   From Lemma~\ref{lemma:GM_convex}, we know that $\textnormal{GM}\left(\{x_i\}_{i\in S} \right) \in \textnormal{Conv}\left(\{x_i\}_{i \in S}\right)$. Thus, owing to Lemma~\ref{lem:dia_cov}, we have  $$\norm{x_j - \textnormal{GM}\left(\{x_i\}_{i\in S} \right)} \leq \max_{k, l \in S} \norm{x_k - x_l},  \, \forall j \in S.$$ Similarly, as $\overline{x}_{S} \in \textnormal{Conv}\left(\{x_i\}_{i \in S}\right)$, we get  
   $$\norm{\textnormal{GM}\left(\{x_i\}_{i\in S} \right) - \overline{x}_{S}} \leq \max_{k, l \in S} \norm{x_k - x_l}.$$ Using these in~\eqref{eqn:before_lemma_cov} we obtain that
   \begin{align*}
       \norm{\textnormal{GM}(x_1, \ldots, x_n) - \overline{x}_{S}} &\leq \frac{n-f}{\sqrt{(n-2f)n}} \max_{i, j \in S} \norm{x_i - x_j} + \max_{i, j \in S} \norm{x_i - x_j} = \left(1+\frac{n-f}{\sqrt{(n-2f)n}}\right) \max_{i, j \in S} \norm{x_i - x_j}.
   \end{align*}
   As $S$ is an arbitrary subset of $[n]$ of size $n - f$, by Definition~\ref{def:rational}, the above proves the proposition.
\end{proof}

\subsection{Coordinate-Wise Median (CWMed)}
\label{sec:CWMed}
For input vectors $x_1, \ldots, \, x_n$, their coordinate-wise median, denoted by $\text{CWMed}(x_1, \ldots, x_n)$, is defined to be a vector whose $k$-th coordinate, for all $k \in [d]$, is defined to be
\begin{equation}
    \left[\text{CWMed}\left(x_1, \ldots, x_n\right)\right]_k \coloneqq \text{Median}\left( [x_1]_k, \ldots [x_n]_k \right). \label{eqn:def_cwmed}
\end{equation}
Before analyzing CWMed, we prove a useful lemma for the median operator.
\begin{lemma}
\label{lemma:median_property}
Consider a set of $n$ real numbers $\{y_1, \ldots, \, y_n\}$. If $f < n/2$ then for any subset $S \subseteq [n]$ with $\card{S} = n -f$ we obtain that
  \begin{equation}
     \card{\left\{i \in S ~ \vline ~ y_i \leq \textnormal{Median}(y_1, \ldots, y_n)\right\}} \geq \frac{n}{2}-f  \quad \text{and} \quad \card{\left\{i \in S ~ \vline ~ y_i \geq \textnormal{Median}(y_1, \ldots, y_n)\right\}} \geq \frac{n}{2}-f.
  \end{equation}
\end{lemma}
\begin{proof}
Consider an arbitrary set $S \subseteq [n]$ with $\card{S} = n-f$. By the definition of the median operator, we have
\[\card{\left\{i \in [n] ~ \vline ~ y_i \leq \textnormal{Median}(y_1, \ldots, y_n)\right\}} \geq \frac{n}{2} \quad \text{and} \quad \card{\left\{i \in S ~ \vline ~ y_i \geq \textnormal{Median}(y_1, \ldots, y_n)\right\}} \geq \frac{n}{2}.\]x
As $\card{S} = n-f > f$, the proof follows immediately from above.
\end{proof}

\noindent \fcolorbox{black}{gainsboro!40}{
\parbox{0.97\textwidth}{\centering
\begin{proposition}
   If $f <n/2$ then \textnormal{CWMed} is an $(f, \, \lambda)$-resilient averaging rule for $\lambda =\frac{n}{2(n-f)} \min \left\{2\sqrt{n-f}, \sqrt{d}\right\}$.
\end{proposition}
}}

\begin{proof}
   Consider a $S \subset [n]$ such that $\card{S} = n-f$. As $f < n/2$, from Lemma~\ref{lemma:median_property} we obtain that
\begin{equation*}
    \min_{i \in S} [x_i]_k\leq \textnormal{Median} \left( [x_1]_k, \ldots [x_n]_k \right) \leq  \max_{i \in S} [x_i]_k.
\end{equation*}
This implies that 
\begin{equation}
    \textnormal{Median} \left( [x_1]_k, \ldots [x_n]_k \right) - (\max_{i \in S} [x_i]_k - \min_{i \in S} [x_i]_k) \leq \min_{i \in S} [x_i]_k. \label{eqn:cw_med_1}
\end{equation}
Note that, by Lemma~\ref{lemma:median_property}, at least $n/2 - f$ values in $\{y_i, ~ i \in S\}$ are greater than or equal to $\textnormal{Median}\left( [x_1]_k, \ldots [x_n]_k \right)$. Thus, as the remaining $n/2$ values in $\{y_i, ~ i \in S\}$ are greater than or equal to $ \min_{i \in S} [x_i]_k$, we obtain that
\begin{align*}
    [\overline{x}_S]_k = \frac{1}{n-f} \sum_{i \in S} [x_i]_k \geq \frac{1}{n-f} \left( \left(\frac{n}{2}-f \right) \textnormal{Median} \left( [x_1]_k, \ldots [x_n]_k \right) + \frac{n}{2}  \min_{i \in S} [x_i]_k\right)
\end{align*}
Substituting from~\eqref{eqn:cw_med_1} above we obtain that
\begin{align}
    &[\overline{x}_S]_k \geq \frac{1}{n-f} \left( \left(\frac{n}{2}-f \right) \textnormal{Median} \left( [x_1]_k, \ldots [x_n]_k \right) + \frac{n}{2}  \left(\textnormal{Median} \left( [x_1]_k, \ldots [x_n]_k \right) - (\max_{i \in S} [x_i]_k - \min_{i \in S} [x_i]_k)\right)\right) \nonumber \\
    &= \textnormal{Median} \left( [x_1]_k, \ldots [x_n]_k \right) - \frac{n}{2(n-f)} (\max_{i \in S} [x_i]_k - \min_{i \in S} [x_i]_k). \label{eqn:cw_med_2}
\end{align}
Similarly, we can show that 
\begin{equation}
    [\overline{x}_S]_k \leq \textnormal{Median} \left( [x_1]_k, \ldots [x_n]_k \right) + \frac{n}{2(n-f)} (\max_{i \in S} [x_i]_k - \min_{i \in S} [x_i]_k). \label{eqn:cw_med_3}
\end{equation}
From~\eqref{eqn:cw_med_2} and~\eqref{eqn:cw_med_3} we obtain that
\begin{equation*}
    \absv{ [\overline{x}_S]_k - \textnormal{Median} \left( [x_1]_k, \ldots [x_n]_k \right)} \leq \frac{n}{2(n-f)} \max_{i,j \in S} ([x_i]_k - [x_j]_k).
\end{equation*}
Finally, substituting from Lemma~\ref{lemma:diameter} we obtain that
\begin{align*}
    \norm{\overline{x}_S - \text{CWMed} (x_1, \ldots, x_n)} &= \sqrt{\sum_{k\in [d]}\absv{ [\overline{x}_S]_k - \text{Median} \left( [x_1]_k, \ldots [x_n]_k \right)}^2}\\
    &\leq \frac{n}{2(n-f)} \sqrt{ \sum_{k\in [d]} \left(\max_{i, j \in S} \card{[x_i]_k - [x_j]_k} \right)^2}\\ &\leq \frac{n}{2(n-f)} \min \left\{2\sqrt{n-f}, \sqrt{d}\right\}\max_{i, j \in S}  \norm{x_i - x_j}.
\end{align*}
The above concludes the proof.
\end{proof}

\subsection{Centered Clipping (CC)}
\label{sec:CC}
This aggregation rule was proposed by~\cite{Karimireddy2021}. Specifically, given the input vectors $x_1, \ldots, \, x_n \in \R^d$, upon choosing a {\em clipping parameter} $c_\tau \geq 0$, we compute a sequence of vectors $v_0, \ldots, \, v_L$ in $\R^d$ such that for all $l \in [L]$,
\begin{equation*}
    v_l \xleftarrow{} v_{l-1} + \frac{1}{n} \sum_{i \in [n]} (x_i - v_{l-1}) \min \left\{1,\frac{c_\tau}{\norm{x_i - v_{l-1}}} \right\}
\end{equation*}
where $v_0$ may be chosen arbitrary. Then, $\text{CC}(x_1,\ldots,x_n) = v_{L}$.

According to~\citet{Karimireddy2021}, by setting specific values for parameters $c_\tau$ and $L$, CC can satisfy the condition of $(f, \, \lambda)$-resilient averaging for $\lambda = 20\sqrt{10} f/n$ when $f<n/9.7$. However, they rely on extra information that is often not possible in practice. Specifically, the values for parameters $c_\tau$ and $L$ depend on the maximal variance of the honest gradients $\sigma$, and we must also know a bound on the initial estimate error $\condexpect{}{\norm{\overline{x_\mathcal{H}}-v_0}^2}$ where $\overline{x_\mathcal{H}}$ is the average of the honest vectors. Analyzing CC under standard assumptions and without any extra information remains an open question.


\subsection{Comparative Gradient Elimination (CGE)}
\label{sec:CGE}
For input vectors $x_1, \ldots, \, x_n$, let $\tau$ denote a permutation on $[n]$ that sorts the input vectors based on their norm and in non-decreasing order, i.e., $\norm{x_{\tau(1)}}\leq \norm{x_{\tau(2)}} \leq \ldots \leq \norm{x_{\tau(n)}}$. CGE outputs the average of the $n-f$ vectors with smallest norm~\cite{gupta2021byzantine}, i.e., 
\begin{equation*}
    \text{CGE}(x_1,\ldots,x_n) = \frac{1}{n-f} \sum_{i=1}^{n-f} x_{\tau(i)}.
\end{equation*}
In general, CGE is \emph{not} resilient averaging as shown below using a counter-example. 

\begin{proof}[Counter-example]
   Consider input vectors $x_1, \ldots, \, x_n$ and a subset $S \subset [n]$ with $\card{S} = n-f$ such that $x_i = x$ for all $i \in S$ where $\norm{x} > 0$. If $\norm{x_j} < \norm{x}$ for all $j \in S \setminus [n]$, and $\sum_{j \in [n] \setminus S} x_j \neq f \times x$ then
\[\text{CGE}(x_1,\ldots,x_n) = \frac{1}{n-f} \left( \sum_{j \in [n] \setminus S} x_j + (n-2f) \, x \right) \neq x. \]
As $\overline{x_S} = x$ and $\max_{i, \, j \in S}\norm{x_i - x_j} = 0$, from above we obtain that, for all $\lambda \geq 0$,
\[\norm{\text{CGE}(x_1,\ldots,x_n) - \overline{x_S}} = \norm{\text{CGE}(x_1,\ldots,x_n) - x} > 0 = \lambda \max_{i, \, j \in S}\norm{x_i - x_j}.\]
Thus, by Definition~\ref{def:rational}, CGE is \emph{not} resilient averaging.
\end{proof}

\newpage
\section{Additional Information on the Experimental Setup}\label{app:exp_setup}

\subsection{Attacks Simulating Byzantine Behavior}
In the experiments of this paper, we use four state-of-the-art attacks that we refer to as \textit{empire}~\cite{empire}, \textit{little}~\cite{little}, \textit{sign-flipping}~\cite{allen2020byzantine}, and \textit{label-flipping}~\cite{allen2020byzantine}. The first two attacks rely on the same core idea. Let $\zeta$ be fixed a non-negative real number and let $a_t$ be the attack vector at time step $t$. At every time step $t$, all Byzantine workers send $\overline{g_t} + \zeta a_t$ to the server, where $\overline{g_t}$ is an estimate of the true gradient at step $t$. The specific details of these attacks are mentioned below.
\begin{itemize}
    \item \textbf{Fall of Empires.} In this attack, $a_t = - \overline{g_t}$. All Byzantine workers thus send $(1 - \zeta) \overline{g_t}$ at step $t$. In our experiments, we set $\zeta = 1.1$ for \textit{empire}, corresponding to $\epsilon = 0.1$ in the notation of the original paper.
    \item \textbf{Little is Enough.} In this attack, $a_t = - \sigma_t$, where $\sigma_t$ is the opposite vector of the coordinate-wise standard deviation of $\overline{g_t}$. In our experiments, we set $\zeta = 1$ for \textit{little}.
\end{itemize}
The remaining attacks rely on different primitives. Specifically, they are defined as follows.
\begin{itemize}
    \item \textbf{Sign-flipping.} In this attack, every Byzantine worker sends the negative of its gradient to the server.
    \item \textbf{Label-flipping.} In this attack, every Byzantine worker computes its gradient on flipped labels before sending it to the server. Since the labels for MNIST, Fashion-MNIST, and CIFAR-10 are in $\{0, 1, ..., 9\}$, the Byzantine workers flip the labels by computing $l' = 9 - l$ for every training datapoint of the batch, where $l$ is the original label and $l'$ is the flipped/modified label.
\end{itemize}

\subsection{Dataset Pre-processing}
MNIST receives an input image normalization of mean $0.1307$ and standard deviation $0.3081$.
Fashion-MNIST is horizontally flipped.
CIFAR-10 is horizontally flipped and we apply a per-channel normalization with means $0.4914, 0.4822, 0.4465$ and standard deviations $0.2023, 0.1994, 0.2010$.

\subsection{Detailed Model Architecture}\label{app:model_arch}
In this section, we discuss the different models tested in our experimental study. In particular, we experimented with one \textit{convolutional} model and one simple \textit{feed-forward neural network} for both MNIST and Fashion-MNIST, as well as one \textit{convolutional} model for CIFAR-10. In order to present the architecture of the different models, we use the compact notation introduced in \cite{distributed-momentum}.

\noindent \fcolorbox{black}{gainsboro!40}{
\parbox{0.98\textwidth}{
L(\#outputs) represents a \textbf{fully-connected linear layer}, R stands for \textbf{ReLU activation}, S stands for \textbf{log-softmax}, C(\#channels) represents a \textbf{fully-connected 2D-convolutional layer} (kernel size 3, padding 1, stride 1), M stands for \textbf{2D-maxpool} (kernel size 2), B stands for \textbf{batch-normalization}, and D represents \textbf{dropout} (with fixed probability 0.25).
}}

\paragraph{Convolutional Model for CIFAR-10.}
The convolutional model used for CIFAR-10, introduced in~\cite{little}, can thus be written in the following way:
\begin{center}
(3,32×32)-C(64)-R-B-C(64)-R-B-M-D-C(128)-R-B-C(128)-R-B-M-D-L(128)-R-D-L(10)-S.
\end{center}

\paragraph{Convolutional Model for (Fashion-)MNIST.}
We adopt the same notation introduced earlier, with the only difference that C(\#channels) now represents a fully-connected 2D-convolutional layer of kernel size 5, padding 0, and stride 1. The convolutional model we used for MNIST and Fashion-MNIST can thus be written in the following way:
\begin{center}
C(20)-R-M-C(20)-R-M-L(500)-R-L(10)-S.
\end{center}

\paragraph{Simple Feed-forward Network for (Fashion-)MNIST.}
We consider a feed-forward neural network composed of two fully-connected linear layers of respectively 784 and 100 inputs (for a total of $d = 79\,510$ parameters) and terminated by a \emph{softmax} layer of 10 dimensions. ReLU is used between the two linear layers. For this particular model, we used the Cross Entropy loss, a total number of workers $n = 15$, a constant learning rate $\gamma = 0.5$, and a clipping parameter $C=2$. We also add an $\ell_2$-regularization factor of $10^{-4}$. Note that some of these constants are reused from the literature on BR, especially from~\cite{little, empire, distributed-momentum}.
\section{Additional Experimental Results}\label{app:exp_results}


\subsection{Results on Fashion-MNIST}\label{app:exp_results_FM}
\begin{figure*}[!ht]
    \centering
    \includegraphics[width=\textwidth]{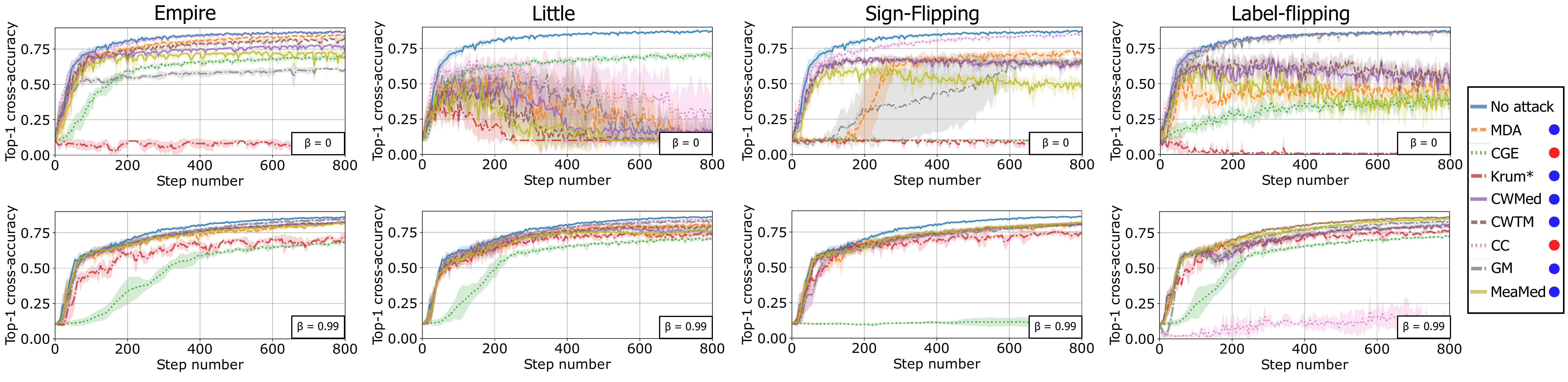}
    \caption{The $1$st and $2$nd rows correspond to experiments performed on Fashion-MNIST with $\beta = 0$ and $\beta = 0.99$, respectively. The different columns show the performance of the learning under the \textit{empire}, \textit{little}, \textit{sign-flipping}, and \textit{label-flipping} attacks with $f = 5$ Byzantine workers.}
    \label{fig:plots_fashion}
\end{figure*}
We also perform experiments (similar to those described in Section~\ref{sec:experiments}) on the Fashion-MNIST dataset. In Figure~\ref{fig:plots_fashion}, we display the top-1 cross accuracies achieved by different aggregation rules on the Fashion-MNIST dataset in a distributed system comprising $n=15$ workers, out of which $f=5$ are Byzantine executing four different state-of-the-art attacks. We compare the performances under two momentum settings: $\beta = 0$ (i.e., momentum is not used) and $\beta = 0.99$.\\
We can clearly see from Figure~\ref{fig:plots_fashion} the improvement that momentum brings to the learning in every single Byzantine setting (i.e., in each of the four attack scenarios), especially for the six resilient averaging aggregation rules (MDA, CWTM, CWMed, MeaMed, Krum$^*$, and GM). However, the performance of CGE seems unaffected by the increase in momentum especially under the \textit{empire}, \textit{little}, and \textit{sign-flipping} attacks. Furthermore, CC displays poor performance under \textit{little} for $\beta = 0$ and under \textit{label-flipping} for $\beta = 0.99$, indicating that there always seems to exist at least one setting where CC (and CGE) display poor performance. All these observations clearly echo the main takeaway of our experiments in Section~\ref{sec:experiments}, where using \textbf{both} a $(f, \lambda)$-resilient averaging aggregation rule together with momentum seems to be crucial to mitigate the effect of Byzantine workers and dramatically improve the learning in an arbitrary adversarial setting (i.e., when the executed attack is not known beforehand).

\subsection{The case of CC - $\beta = 0.9$}\label{app:exp_results_CC}
\begin{figure*}[!ht]
    \centering
    \includegraphics[width=\textwidth]{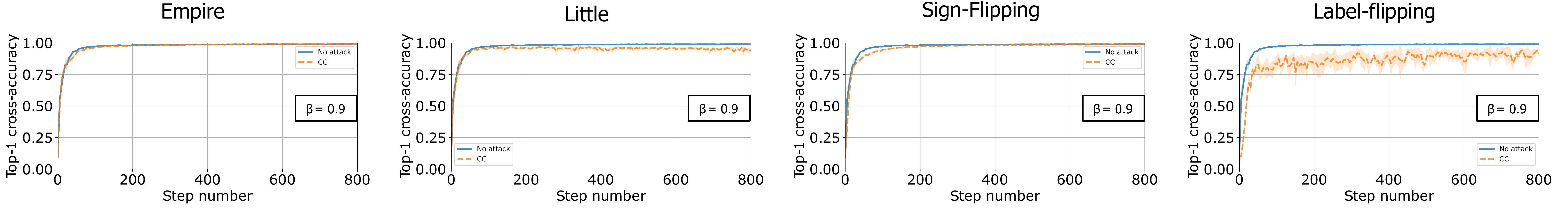}
    \includegraphics[width=\textwidth]{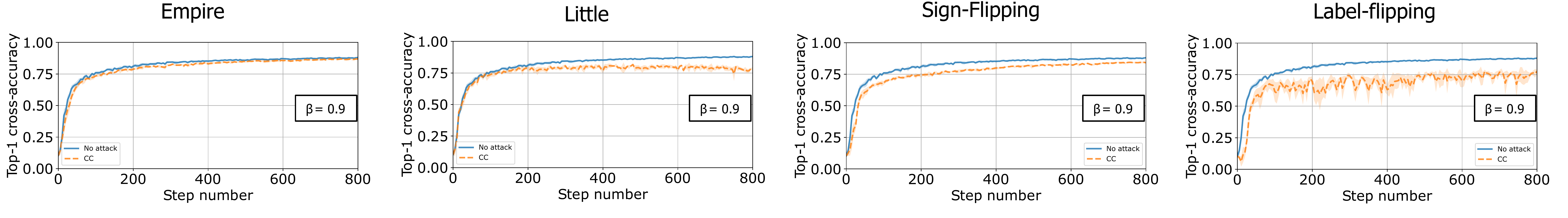}
    \caption{The $1$st and $2$nd rows correspond to experiments performed on MNIST and Fashion-MNIST, respectively, using $\beta = 0.9$ and the CC aggregation rule. The different columns show the performance of the learning under the \textit{empire}, \textit{little}, \textit{sign-flipping}, and \textit{label-flipping} attacks with $f = 5$ Byzantine workers. The ``No attack'' curve is also provided as a baseline for comparison. }
    \label{fig:plots_mnist_CC_0.99}
\end{figure*}
In Figure~\ref{fig:plots_mnist_CC_0.99}, we show the performance of CC (which is not an $(f, \lambda)$-resilient averaging rule) on the MNIST and Fashion-MNIST datasets, with $\beta = 0.9$ and $f = 5$ Byzantine workers. CC displays good performance against all four attacks for that particular value of $\beta$. Essentially, CC can consistently work for some values of momentum ($\beta = 0.9$), while others significantly deteriorate its performance in some cases (see $\beta = 0.99$ in Figure~\ref{fig:plots} of the main paper). Precisely characterizing the impact of momentum on CC's performance remains arguably an open question.

\subsection{Results on MNIST With 7 Byzantine Workers}\label{app:exp_results_M_0.999}
\begin{figure*}[!ht]
    \centering
    \includegraphics[width=\textwidth]{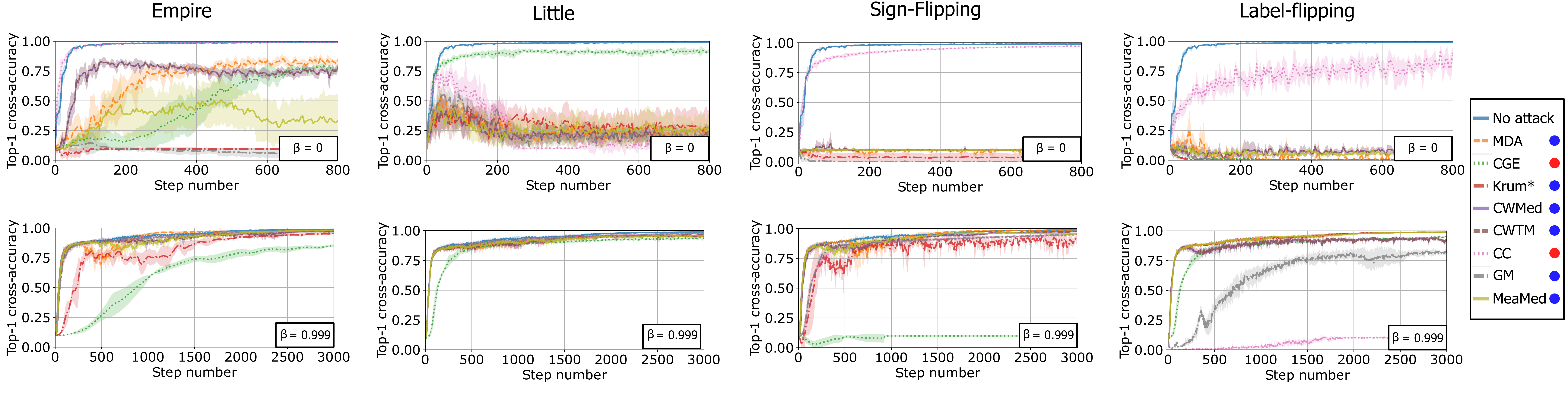}
    \caption{The $1$st and $2$nd rows correspond to experiments performed on MNIST with $\beta = 0$ and $\beta = 0.999$, respectively. The different columns show the performance of the learning under the \textit{empire}, \textit{little}, \textit{sign-flipping}, and \textit{label-flipping} attacks with $f = 7$ Byzantine workers. Note that in this experiment, we use Multi-Krum$^*$ with $q = n - f$ instead of Krum$^*$.}
    \label{fig:plots_mnist_0.999}
\end{figure*}
In this paragraph, we present some learning performances on the MNIST dataset in four adversarial settings where $f = 7$ out of 15 workers are Byzantine. It turns out that in such an extreme adversarial scenario where $f$ reaches the maximum tolerable value of $\floor{\frac{n}{2}}$, an even larger value of $\beta$, and thus more learning steps, are needed to guarantee a good performance in the presence of Byzantine workers. In Figure~\ref{fig:plots_mnist_0.999}, we consider two values for $\beta$ ($0$ and $0.999$), and showcase the advantages of using momentum in such a setting.\\
The observations to be made here are very similar to the ones already stated in Sections~\ref{sec:experiment-results} and \ref{app:exp_results_FM}. In a few words, we can clearly see that setting $\beta$ to 0.999 improves the top-1 cross-accuracies of all six $(f, \lambda)$-resilient averaging aggregation rules. However, for the two non-resilient averaging rules (CC and CGE), momentum need not improve the learning.

\end{document}